\documentclass[11pt]{article}

\usepackage{blindtext}

\usepackage[preprint]{jmlr2e}

\newcommand\fcal{\mathcal{F}}
\newcommand\rbb{\mathbb{R}}

\newcommand\nbb{\mathbb{N}}
\newcommand\ebb{\mathbb{E}}
\newcommand\pbb{\mathbb{P}}
\newcommand\ybb{\mathbb{Y}}

\newcommand\kongge{\hspace*{15pt}}

\newcommand\diam{\mathrm{diam}\ }

\newcommand\pscr{\mathscr{P}}
\newcommand\bscr{\mathscr{B}}
\newcommand\fscr{\mathscr{F}}
\newcommand\kscr{\mathscr{K}}
\newcommand\dbb{\mathbb{D}}
\newcommand\sbb{\mathbb{S}}
\newcommand\bbb{\mathbb{B}}
\newcommand\dscc{\overline{\mathrm{conv}(\pm\dbb)}}
\newcommand\pscc{\overline{\mathrm{conv}(\pm\pbb_s)}}
\newcommand\tscr{\mathscr{T}}

\newcommand\escr{\mathscr{E}}
\newcommand\gcal{\mathcal{G}}
\newcommand\starhull{\mathrm{star}}
\newcommand\ncal{\mathcal{N}}
\newcommand\pcal{\mathcal{P}}
\newcommand\qbb{\mathbb{Q}}
\newcommand{\wb}{{\bm{w}}}
\newcommand{\xb}{{\bm{x}}}
\newcommand{\yb}{{\bm{y}}}

\newcommand{\ub}{{\bm{u}}}
\newcommand{\vb}{{\bm{v}}}
\newcommand{\hcal}{\mathcal{H}}
\newcommand{\fcalw}{\widetilde{\fcal}_{p,\nu,s}}
\newcommand{\fcalwd}{\widetilde{\fcal}_{p,\tau_d,s}}
\newcommand{\etab}{\boldsymbol{\eta}}

\newcommand{\xib}{\boldsymbol{\xi}}
\newcommand{\thetab}{\boldsymbol{\theta}}
\usepackage{tikz}
\usetikzlibrary{arrows.meta, positioning, shapes.geometric, calc, decorations.pathreplacing}
\usepackage{amsmath}
\usepackage{mathrsfs}
\usepackage{bm}
\usepackage{multirow}
\usepackage{longtable}
\usepackage[normalem]{ulem}

\firstpageno{1}
\begin{document}

\title{Shallow ReLU$^s$ Networks in $L^p$-Type and Sobolev Spaces: Approximation and Path-Norm Controlled Generalization\thanks{Authors are listed in alphabetical order and contributed equally to this work. Corresponding authors: Fanghui Liu and Lei Shi.}}

\author{\name Weizhao Li \email weizhaoli24@m.fudan.edu.cn \\
       \addr School of Mathematical Sciences\\ 
       Fudan University\\
       Shanghai, China
       \AND
       \name Fanghui Liu \email fanghui.liu@sjtu.edu.cn \\
       \addr School of Mathematical Sciences\\
       Institute of Natural Sciences and MOE-LSC\\ 
       Shanghai Jiao Tong University\\
       Shanghai, China
       \AND
       \name Lei Shi \email leishi@fudan.edu.cn \\
       \addr School of Mathematical Sciences\\ 
       Shanghai Key Laboratory for Contemporary Applied Mathematics\\
       Fudan University\\
       Shanghai, China}

\maketitle

\begin{abstract}
This paper studies approximation by shallow ReLU$^s$ networks, $\sigma_s(t)=\max\{0,t\}^s$, together with their generalization behavior under $\ell_1$ path-norm control. For the $L^p$-type integral spaces $\widetilde{\mathcal{F}}_{p,\tau_d,s}$, $1\le p\le2$, spherical harmonic analysis yields approximation bounds for shallow networks. In particular, when $\tau_d$ is the uniform measure and $1\le p<2$, the approximation rate is $O\!\left(m^{-\frac{p(2s+2d+1)-2d}{2dp}}\right)$ for $1\le p\le p^*$ and $O\!\left(m^{-\frac{p(4s+3d-1)-2d+2}{4dp}}\right)$ for $p^*<p<2$, where $p^*=\frac{2d+2}{d+3}$. Approximation bounds for Sobolev spaces $W^{\alpha,p}$, $1\le p<2$, are obtained through embeddings into spectral Barron spaces. For nonparametric regression with sub-Gaussian noise, path-norm-regularized shallow ReLU$^s$ networks achieve minimax-optimal rates $O\!\left(n^{-\frac{d+2s+1}{2d+2s+1}}\log n\right)$ over $\mathscr{B}_s$ and $O\!\left(n^{-\frac{2\alpha}{2\alpha+d}}\log n\right)$ over $W^{\alpha,\infty}$, with matching lower bounds up to logarithmic factors.
\end{abstract}


\section{Introduction}\label{sec:intro}

Deep learning has shown remarkable effectiveness in high-dimensional approximation problems, particularly in scientific computing, inverse problems, and operator learning \citep{Han_2018,adcock2021deepneuralnetworkseffective,ChristianBeck23}. In many such settings, the ReLU$^s$ activation $\sigma_s(t)=\max\{0,t\}^s$ ($s\in\nbb_0$)\label{sym:first:sigma_s} is especially relevant because it yields piecewise-polynomial representations that are well suited to smooth targets and derivative-sensitive tasks \citep{Yang2024,he2024expressivityapproximationpropertiesdeep}.

To obtain meaningful approximation guarantees, one must first identify an appropriate function class. Reproducing kernel Hilbert spaces (RKHSs) are a standard choice, but they cannot effectively approximate even a single ReLU neuron \citep{yehudai2019power}. More fundamentally, \cite{Steinwart24} proved that on an uncountable compact space, no RKHS contains all continuous functions. This result highlights an intrinsic limitation of RKHSs: they are naturally adapted to smooth functions and therefore have limited expressive power for nonsmooth targets. Sobolev spaces $W^{\alpha,p}$, by contrast, accommodate such nonsmooth functions and are central in approximation theory and the analysis of partial differential equations \citep{Triebel1983}. However, Sobolev spaces are also too large from a statistical perspective: \emph{any} machine learning method is statistically ineffective over the full Sobolev class \citep{Wainwright_2019}. This raises a basic theoretical question: \emph{which function classes can neural networks learn efficiently, and how does that answer depend on the activation function and the regularity of the target space?} \cite{e2021barronspaceflowinducedfunction} reformulated Barron spaces \citep{Barron93} and showed that, for shallow ReLU networks, the Barron space is exactly the largest function class that can be learned statistically efficiently.\footnote{If a function belongs to the Barron space, then shallow ReLU networks can learn it without suffering from the curse of dimensionality; conversely, if a function can be sufficiently learned by shallow networks without the curse of dimensionality, then it must belong to the Barron space.}

The connection between RKHSs and Barron spaces can be made precise through the $L^p$-type spaces studied by \cite{Celentano2021MinimumCI} and \cite{chen2025dualityframeworkanalyzingrandom}. These works consider the following infinite-width ReLU$^s$ model:
\begin{align*}
    f(\xb)=\int_{\mathcal{V}}a(\wb,b)\sigma_s\left(\wb^\top \xb+b\right)\mathrm{d} \nu(\wb,b)\,,
\end{align*}
where $\mathcal{V}=\sbb^{d-1}\times[-1,1]$\label{sym:first:sphere} is the parameter space for $(\bm w,b)$ and $\mathrm{d}\nu$ is a probability measure on $\mathcal{V}$. If $a(\wb,b)\in L^p(\mathrm{d}\nu)$, then $f$ belongs to an $L^p$-type space, denoted by $f\in\fcalw$ (with formal definition in Section~\ref{sec:preliminary}). In particular, $\fcalw$ reduces to an RKHS when $p=2$, while the Barron space is close to a variational space \citep{bach2016breakingcursedimensionalityconvex} and can be identified with the union of a family of $L^p$-type spaces \citep{chen2025dualityframeworkanalyzingrandom,Celentano2021MinimumCI}.
By H\"{o}lder's inequality, for $1\le p\le2$,
\[
\widetilde{\fcal}_{2,\nu,s}\subseteq\widetilde{\fcal}_{p,\nu,s}\subseteq\widetilde{\fcal}_{1,\nu,s}\,,
\]
so these spaces form a natural bridge from RKHSs to Barron spaces. 
The connections among these spaces can be conceptually described by Figure~\ref{fig:space-hierarchy}. The $L^p$-type spaces describe the size of the neural-network integral representation, whereas Sobolev spaces describe classical smoothness. The two are connected through embedding theorems: suitable Sobolev regularity implies membership in a spectral Barron or related $L^p$-type representation space, which then yields shallow-network approximation rates.
Precise definitions of these function spaces are given in Section~\ref{sec:preliminary}.

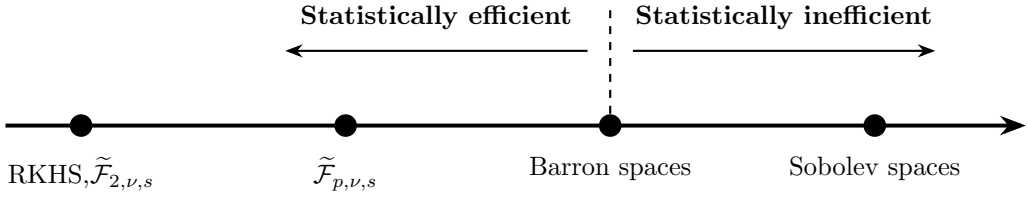
\begin{figure}[t]
    \centering
    \begin{tikzpicture}[
        every node/.style={font=\small},
        dot/.style={circle, fill, inner sep=3pt},
        label above/.style={above=8pt, align=center, font=\small},
        label below/.style={below=8pt, align=center, font=\small}
    ]
        \draw[-{Stealth}, line width=1.5pt] (0,0) -- (13.5,0);

        \node[dot] (n1) at (1,0)   {};
        \node[dot] (n2) at (4.5,0) {};
        \node[dot] (n3) at (8,0)   {};
        \node[dot] (n4) at (11.5,0){};

        \node[label below] at (n1) {RKHS,$\widetilde{\fcal}_{2,\nu,s}$};
        \node[label below] at (n3) {Barron spaces};

        \node[label below] at (n2) {$\widetilde{\fcal}_{p,\nu,s}$};
        \node[label below] at (n4) {Sobolev spaces};

        \draw[dashed,thick] (n3) -- ++(0, 1.6);

        \draw[{Stealth}-, thick] ($(n2)+(-0.8,0.98888)$) -- ($(n3)+(-0.3,0.98888)$)
            node[midway, above=4pt, font=\small\bfseries] {Statistically efficient};

        \draw[-{Stealth}, thick] ($(n3)+(0.3,0.98888)$) -- ($(n4)+(0.8,0.98888)$)
            node[midway, above=4pt, font=\small\bfseries] {Statistically inefficient};
    \end{tikzpicture}
    \caption{Hierarchy of function spaces considered in this paper, ordered from smaller to larger in generality.}
    \label{fig:space-hierarchy}
\end{figure}

We consider approximation of functions in $L^p$-type spaces and Sobolev spaces by shallow ReLU$^s$ networks. For width $m$, let
\begin{equation*}\label{sym:first:Sigma_ms}
    \Sigma_{m}^s=\left\{f_m=\frac{1}{m}\sum_{j=1}^m a_j\sigma_s\left(\wb_j^\top \xb+b_j\right):a_j\in\rbb,\,\wb_j\in\rbb^d,\,b_j\in\rbb\right\} \,,
\end{equation*}
and let $\Sigma^s=\bigcup_{m=1}^\infty \Sigma_m^s$ denote the space of shallow ReLU$^s$ networks of arbitrary width.
Studying approximation over $\fcalw$ clarifies how the difficulty of learning changes along the spectrum from RKHSs to Barron spaces. For Sobolev spaces, the theory is relatively complete when $p\ge2$, whereas the regime $1\le p<2$ remains much less understood. This regime is important in PDE analysis and inverse problems with nonsmooth or singular structure: solutions with shocks, edges, or sparsity often belong naturally to $W^{\alpha,p}$ for $1\le p<2$, but need not lie in the Hilbert-type space $W^{\alpha,2}$. Extending approximation guarantees to this setting is therefore relevant both theoretically and practically.

\subsection{Related Work}

We first discuss previous work on approximation over $L^p$-type spaces and Barron spaces, and then give an overview of generalization results for shallow ReLU$^s$ networks related to such spaces.

For approximation over $L^p$-type spaces, the existing literature covers only the two endpoint cases $p=1$ and $p=2$ \citep{bach2016breakingcursedimensionalityconvex}. In the intermediate regime $1<p<2$, \cite{chen2025dualityframeworkanalyzingrandom} studied random feature models over $\fcalw$ through a dual framework linking approximation and estimation, and obtained the approximation upper bound $O(m^{-1+\frac{1}{p}}\log^{\frac{3}{2}}m)$\label{sym:first:bigO}. This result concerns random feature models, where the inner parameters are fixed after sampling, and does not directly give approximation rates for fully optimized shallow networks. The endpoint theory for $\widetilde{\fcal}_{1,\nu,s}$ was developed in \cite{bach2016breakingcursedimensionalityconvex} and \cite{siegel2025optimalapproximationzonoidsuniform} using zonoid and zonotope methods, but those arguments rely heavily on the equivalent integral representation \eqref{eq:interform}, which is difficult to extend to general $p$.

In Sobolev approximation, a central challenge is to measure the regularity of the target in an $L^q$ norm while deriving error bounds in an $L^p$ norm. When $q\ge p$, optimal approximation rates are known; when $q<p$, the problem becomes substantially harder and the available results are limited. Using the embedding $W^{s+\frac{d+1}{2},2}\hookrightarrow\bscr_s$, \cite{MaoSiegelXu2026,Siegel_2022sharp} established the estimate
\begin{align*}
    \inf_{f_m\in\Sigma_{m}^s}\|f-f_m\|_{L^\infty}\lesssim \|f\|_{W^{s+\frac{d+1}{2},2}}m^{-\frac{1}{2}-\frac{2s+1}{2d}}\,.
\end{align*}
This result shows that regularity measured in $L^2$ can still yield an optimal approximation rate in $L^\infty$, providing the first result in the regime $p>q$. However, it covers only the parameter setting $q=2$, $p=\infty$, and $\alpha=s+\frac{d+1}{2}$; other regimes remain open. In particular, when $\alpha>s+\frac{d+1}{2}$, one can derive only the upper bound $O\!\left(m^{-\frac{1}{2}-\frac{2s+1}{2d}}\right)$ by combining embeddings with known approximation results on $\bscr_s$. More precisely, for $2\le p\le\infty$, $s\in\nbb_0$, $\alpha>0$, and $f\in W^{\alpha,p}$, one has
\begin{align*}
    \inf_{f_m\in\Sigma_{m}^s}\|f-f_m\|_{L^p}\lesssim\|f\|_{W^{\alpha,p}}\begin{cases}
        m^{-\frac{1}{2}-\frac{2s+1}{2d}}\,,&\alpha\ge s+\frac{d+1}{2}\,,\\
        m^{-\frac{\alpha}{d}}\,,&\alpha\in\nbb\cap\left(0,s+\frac{d+1}{2}\right)\,.
    \end{cases}
\end{align*}
When $\alpha<s+\frac{d+1}{2}$, achieving the optimal rate requires $\alpha$ to be an integer, and the noninteger case remains open. Moreover, the available theory applies only to $2\le p\le\infty$, mainly because \cite{MaoSiegelXu2026} relies on the spectral Barron space $\bscr_s$ as an intermediate space, while the embedding $W^{\alpha,p}\hookrightarrow\bscr_s$ generally fails for $1\le p<2$. We address this regime by establishing new embeddings between Sobolev spaces and spectral Barron spaces.

There is also substantial literature on the generalization behavior of shallow ReLU$^s$ networks. \cite{bach2016breakingcursedimensionalityconvex} and \cite{Parhi_2023} systematically analyzed infinite-width ReLU models. In particular, \cite{bach2016breakingcursedimensionalityconvex} proved a representation theorem for integral ReLU networks with variational-norm regularization, implying that a finite-width network $f_m$ can always be trained, and established the generalization upper bound $O\!\left(n^{-\frac{1}{2}}\right)$ over the variational space $\widetilde{\fcal}_{1,\nu,1}$. \cite{Parhi_2023} used Barron regularization to show that the minimax-optimal rate over Barron spaces is $O\!\left(n^{-\frac{d+3}{2d+3}}\right)$. For finite-width networks, \cite{liu2024learningnormconstrainedoverparameterized} explored the trade-off between sample complexity and dimension dependence, and obtained the upper bound $O\!\left(n^{-\frac{d+2}{2d+2}}\right)$ over Barron spaces under $\ell_1$ path norm constraints, although that rate is not minimax optimal but the dimension dependence is relaxed. By contrast, \cite{yang2024nonparametricregressionusingoverparameterized} considered $\ell_2$ path norms and proved minimax-optimal rates $O\!\left(n^{-\frac{d+3}{2d+3}}\right)$ and $O\!\left(n^{-\frac{2\alpha}{2\alpha+d}}\right)$ over Barron and H\"older spaces, respectively. Since the $\ell_2$ path norm leads to a formulation equivalent to $\|\cdot\|_2$-regularized learning \citep{yang2024nonparametricregressionusingoverparameterized}, its analysis is substantially cleaner than that of the $\ell_1$ counterpart.

\subsection{Contributions}

The paper proves approximation results over $L^p$-type spaces $\fcalw$, extends the Sobolev approximation theory to $W^{\alpha,p}$ when $1\le p<2$, and establishes minimax-optimal generalization bounds over Barron and Sobolev spaces. The main contributions are as follows:
\begin{itemize}
    \item Using spherical harmonic analysis, we establish new approximation results for shallow ReLU$^s$ networks over $L^p$-type spaces $\fcalw$. To the best of our knowledge, this is the first result that gives approximation rates for the intermediate $L^p$-type scale $1<p<2$ using shallow networks. When the parameter measure is the uniform distribution $\tau_d$ and $1\le p<2$, Theorem~\ref{thm:var-approx} gives the two-regime rate $O\!\left(m^{-\frac{p(2s+2d+1)-2d}{2dp}}\right)$ for $1\le p\le p^*$ and $O\!\left(m^{-\frac{p(4s+3d-1)-2d+2}{4dp}}\right)$ for $p^*<p<2$, where $p^*=\frac{2d+2}{d+3}$. 
    \item Using embeddings between spectral Barron spaces and Sobolev spaces, we derive approximation rates for shallow ReLU$^s$ networks over $W^{\alpha,p}$ when $1\le p<2$ and $s\in\nbb$. In particular, when $\alpha\ge s+d$, the rate is $O\!\left(m^{-\frac{d+2s}{2(d+1)}}\right)$; when $\alpha\in\nbb\cap(0,s+d)$, the rate is $O\!\left(m^{-\frac{\alpha(d+2s)}{2(s+d)(d+1)}}\right)$.
    \item We prove that for any $s\in\nbb_0$, shallow ReLU$^s$ networks with $\ell_1$ path norm constraints achieve minimax-optimal generalization bounds over both Barron and Sobolev spaces. Our analysis builds upon techniques from \citep{yang2024nonparametricregressionusingoverparameterized,Wainwright_2019}. The framework of \cite{yang2024nonparametricregressionusingoverparameterized} applies only to bounded or Gaussian noise; here we extend the generalization bounds to sub-Gaussian noise distributions using generic chaining tail bounds \citep{Dirksen15}.
\end{itemize}

The rest of the paper is organized as follows. Section~\ref{sec:preliminary} reviews the shallow-network model and the relevant Barron, $L^p$-type, and Sobolev spaces. Section~\ref{sec:approximation} presents the main approximation results. Section~\ref{sec:generalization} establishes minimax-optimal generalization bounds. Section~\ref{sec:conclusion} concludes the paper. Appendix~\ref{appendix:symbols} collects the notation used throughout.

\section{Preliminaries}\label{sec:preliminary}
We recall the function spaces and learning framework used in the approximation and generalization results.

\subsection{Function Spaces}
\paragraph{Barron spaces} Consider functions of the following form:
\begin{align*}
    f_\rho(\xb)=\ebb_{(a,\wb,b)\sim\rho}\left[a\sigma_s\left(\wb^\top\xb+b\right)\right],\quad \xb\in\Omega\,,
\end{align*}
where $\mathrm{d}\rho$ is a probability measure on the Borel $\sigma$-algebra. These are precisely infinite-width shallow ReLU$^s$ networks. Define the \emph{$s$th-order Barron space} by
\begin{align*}
    \bscr_s=\{f=f_\rho:\rho\text{ is a probability measure}\}\,,
\end{align*}
\label{sym:first:bscr_s}
and define the \emph{$s$th-order Barron norm} of $f_\rho$ by
\begin{align*}
    \|f_\rho\|_{\bscr_s}=\inf\left\{\ebb_{(a,\wb,b)\sim\rho}\!\left[|a|\left(\|\wb\|_{1}+|b|\right)^s\right]:f=f_\rho\right\}\,.
\end{align*}
\label{sym:first:bscr_norm}
Here, $\bscr_1$ is exactly the Barron space introduced in \cite{e2021barronspaceflowinducedfunction}, for which optimal approximation rates and metric-entropy estimates are already known \citep{Siegel_2022sharp}. For general $s\in\nbb_0$, Appendix~\ref{appendix:proof-barron-approx} identifies $\bscr_s$ with the variation space defined in \cite{Siegel_2022sharp}; hence the same approximation and entropy estimates apply.

\begin{lemma}\label{thm:barron-approx}
    Let $s\in\nbb_0$. Then for any $f\in\bscr_s$, one has
    \begin{align*}
        \inf_{f_m\in\Sigma_m^s}\|f-f_m\|_{L^\infty}\lesssim\|f\|_{\bscr_s}\,m^{-\frac{1}{2}-\frac{2s+1}{2d}}\,.
    \end{align*}
\label{sym:first:lesssim}
\end{lemma}

The proof of this lemma relies on the equivalence between the Barron space and the variation space, which allows us to invoke the corresponding known results directly. Details are given in Appendix~\ref{appendix:proof-barron-approx}.

\paragraph{$L^p$-type spaces}\label{sym:def:fcalw}
The study of networks in integral form can be traced back to kernel methods and random feature models. If a signed Radon measure admits a density $a(\wb,b)$ with respect to a fixed probability measure $\nu$, then the collection of all functions satisfying
$$
\int_{\mathcal{V}}|a(\wb,b)|^2\,\mathrm{d}\nu(\wb,b)<\infty
$$
forms a reproducing kernel Hilbert space (RKHS) \citep{bach2016breakingcursedimensionalityconvex,Celentano2021MinimumCI,chen2025dualityframeworkanalyzingrandom}, with reproducing kernel
\begin{align*}
    k(\xb,\xb')=\int_{\mathcal{V}}\sigma_s\left(\wb^\top\xb+b\right)\sigma_s\left(\wb^\top\xb'+b\right)\mathrm{d}\nu(\wb,b)\,.
\end{align*}
Although kernel methods are mathematically well developed, they suffer from the curse of dimensionality. For example, they struggle to learn ridge functions that depend on only a few input coordinates. To mitigate this limitation, one can generalize the RKHS construction. Since an RKHS is essentially a weighted $L^2$ space, \cite{bach2016breakingcursedimensionalityconvex} introduced the corresponding weighted $L^1$ space in the ReLU setting, namely the class of functions satisfying $\int_{\mathcal{V}}|a(\wb,b)|\,\mathrm{d}\nu(\wb,b)<\infty$, often called convex neural networks. \cite{Celentano2021MinimumCI,chen2025dualityframeworkanalyzingrandom} then considered the general case $p\ge1$, leading to the family of $L^p$-type spaces
\begin{align*}
    \fcalw=\left\{f(\xb)=\int_{\mathcal{V}}a(\wb,b)\sigma_s\left(\wb^\top\xb+b\right)\mathrm{d}\nu(\wb,b):a(\wb,b)\in L^p(\mathrm{d}\nu)\right\}\,,
\end{align*}
equipped with the norm
\begin{align*}
    \|f\|_{\fcalw}=\inf\left\{\|a(\wb,b)\|_{L^p(\mathrm{d}\nu)}:f(\xb)=\int_{\mathcal{V}}a(\wb,b)\sigma_s\left(\wb^\top\xb+b\right)\mathrm{d}\nu(\wb,b)\right\}\,.
\end{align*}
\label{sym:first:fcalw_norm}
For a finite network
\[
    f_m(\xb)=\frac{1}{m}\sum_{j=1}^m a_j\sigma_s\left(\wb_j^\top\xb+b_j\right)\,,
\]
the corresponding discrete control is the generalized $\ell_1$ path norm
\begin{align*}
    \|f_m\|_{\pscr_{1,s}}=\frac{1}{m}\sum_{j=1}^m|a_j|\left(\|\wb_j\|_1+|b_j|\right)^s\,.
\end{align*}
\label{sym:first:pscr_s}
This quantity is the atomic, finite-width analogue of the Barron or variation norm: replacing the finite sum by a signed measure gives the integral representation underlying Barron spaces. This link motivates the path-norm control used in the generalization analysis below.

\paragraph{Sobolev spaces}
Sobolev spaces play a central role in areas such as partial differential equations. We use the standard definition based on weak derivatives. For $1\le p\le\infty$, the \emph{Sobolev space} of order $\alpha$ is defined by
\begin{align*}
    W^{\alpha,p}(\rbb^d)=\left\{f\in L^p(\rbb^d):\|f\|_{W^{\alpha,p}(\rbb^d)}<\infty\right\}\,,
\end{align*}
\label{sym:first:Sobolev}
where:
\begin{enumerate}
    \item When $\alpha$ is an integer,
    \begin{align*}
        \|f\|_{W^{\alpha,p}(\rbb^d)}=\sum_{|\beta|\le \alpha}\|\partial^\beta f\|_{L^p(\rbb^d)}\,;
    \end{align*}
    \label{sym:first:Sobolev_norm}
    \item When $\alpha$ is not an integer, write $\alpha=k+\theta$, where $k$ is a nonnegative integer and $\theta\in(0,1)$. Then
    \begin{align*}
        \|f\|_{W^{\alpha,p}(\rbb^d)}=\sum_{|\beta|\le k}\|\partial^\beta f\|_{L^p(\rbb^d)}+\sum_{|\beta|=k}\left\|\frac{\partial^\beta f(x)-\partial^\beta f(y)}{\|x-y\|_2^{\frac{d}{p}+\theta}}\right\|_{L^p(\rbb^d\times\rbb^d)}\,.
    \end{align*}
\end{enumerate}
Restricting $W^{\alpha,p}(\rbb^d)$ to $\Omega$ yields the corresponding restricted Sobolev space $W^{\alpha,p}(\Omega)$.

\subsection{Background: Generalization Analysis}

Unless otherwise specified, the input domain is $\Omega=\bbb^d$\label{sym:BBBd}; the arguments also extend to compact sets with Lipschitz boundary. The label space is $Z\subseteq\rbb$. Let $(X,Y)$ follow the true distribution $\mathrm{d}\rho$ on $\Omega\times Z$, and suppose that we observe $n$ i.i.d. noisy samples $D_n=\{(\xb_i,y_i)\}_{i=1}^n$. The goal of the \emph{regression problem} is to learn, from $D_n$, a function $f_m\in\Sigma^s$ that approximates the regression function
\begin{align*}
    f_\rho(\xb)=\ebb[Y\mid X=\xb]=\int_Y y\,\mathrm{d}\rho(y\mid\xb),\quad \xb\in\Omega\,,
\end{align*}
where $\rho(\cdot\mid\xb)$ denotes the conditional distribution of $Y$ given $X=\xb$. It is standard to assume that $f_\rho$ belongs to a hypothesis space $\hcal$, such as a Barron space \citep{liu2024learningnormconstrainedoverparameterized} or a Sobolev space \citep{yang2024nonparametricregressionusingoverparameterized,Yang2024}, reflecting prior knowledge about the target.

Besides fitting the data well, neural-network training must also control model complexity in order to ensure good generalization. In the overparameterized regime, the classical viewpoint that measures complexity by parameter count is no longer adequate; the double-descent phenomenon \citep{Belkin_2020,hastie2020surpriseshighdimensionalridgelesssquares} is a standard example. This has motivated a norm-based view of training, in which the optimization procedure is understood as explicitly or implicitly minimizing a suitable weight norm, often called a representation cost \citep{neyshabur2015normbasedcapacitycontrolneural,ongie2019functionspaceviewbounded,liu2024learningnormconstrainedoverparameterized,yang2024nonparametricregressionusingoverparameterized}.

We adopt explicit regularization and consider the following $L^2$ empirical risk minimization problem:
\begin{align*}
    f_m\in\arg\min_{f\in\Sigma^s}\left\{\frac{1}{n}\sum_{i=1}^n \left|f(\xb_i)-y_i\right|^2+\lambda\|f\|\right\}\,,
\end{align*}
where $\lambda\equiv\lambda(n)>0$ is the regularization parameter and $\|\cdot\|$ denotes a weight norm of the network. We choose the (generalized) $\ell_1$ path norm \citep{e2021barronspaceflowinducedfunction,liu2024learningnormconstrainedoverparameterized}:
\begin{align*}
    \|f\|_{\pscr_{1,s}}=\frac{1}{m}\sum_{i=1}^m|a_i|\left(\|\wb_i\|_1+|b_i|\right)^s\,,
\end{align*}
which can be viewed as a discrete analogue of the Barron norm. Denote by $\Sigma_{m,M}^s$ the set of shallow ReLU$^s$ networks with $m$ neurons and $\ell_1$ path norm at most $M$. Note that the path norm is not parametrization invariant: two different parameterizations of the same network function may have different path norms. This dependence on representation is precisely why the path norm provides a meaningful measure of model complexity.

We consider the expected generalization error of the estimator $f_m$,
\[
    \ebb_{D_n}\!\left[\left\|f_m-f_\rho\right\|_{L^2(\mathrm{d}\mu)}\right]\,,
\]
where $\mathrm{d}\mu$ is the marginal law of $X$. A matching minimax lower bound identifies the resulting estimator as statistically optimal.

\section{Approximation Results}\label{sec:approximation}
This section presents approximation results for shallow ReLU$^s$ networks over the $L^p$-type space $\fcalwd$ and the Sobolev space $W^{\alpha,p}$. The first result treats $L^p$-type integral spaces with optimized inner parameters. The second concerns Sobolev spaces in the range $1\le p<2$, where Hilbert-space methods based on $L^2$ regularity no longer apply directly.
\subsection[Approximation of Lp-type Spaces]{Approximation of $L^p$-type Spaces}

We first study approximation of the integral representation class $\fcalwd$, where $\tau_d$ denotes the normalized uniform measure on the sphere. Throughout this subsection, $\fcalwd$ is understood in the homogenized parameterization
\[
    f(\xb)=\int_{\sbb^d}a(\thetab)\sigma_s((\xb,1)\cdot\thetab)\,d\tau_d(\thetab)\,,
    \qquad \thetab\in\sbb^d\,.
\]
Previous results for the intermediate range $1<p<2$ were mainly obtained for random feature models, in which the inner parameters are sampled and then fixed. Here the inner parameters are optimized as part of the shallow network, and this additional flexibility leads to better approximation rates throughout the range $1\le p<2$. The breakpoint $p^*$ below comes from a change in the spherical-harmonic projection estimate.
\begin{theorem}[Approximation of $L^p$-type spaces by shallow networks]\label{thm:var-approx}
    Let $\tau_d$ be the uniform measure on $\sbb^d$. For $d\ge2$, $s\in\nbb_0$, $1\le p<2$, and $f\in\fcalwd$, the following bound holds for every $m\ge2$:
	    \begin{align*}
	        \inf_{f_m\in\Sigma_m^s}\|f_m-f\|_{L^2}\lesssim\|f\|_{\fcalwd}\begin{cases}
	            m^{-\frac{p(2s+2d+1)-2d}{2dp}}\,,&1\le p\le p^*\,,\\
	            m^{-\frac{p(4s+3d-1)-2d+2}{4dp}}\,,&p^*<p<2\,.
	        \end{cases}
	    \end{align*}
    where $p^*=\frac{2d+2}{d+3}$.
\end{theorem}

The approximation error in Theorem~\ref{thm:var-approx} decreases as $p$ increases. This agrees with the nesting $\widetilde{\fcal}_{q,\tau_d,s}\subset \widetilde{\fcal}_{p,\tau_d,s}$ for $1\le p<q\le2$, so the target class becomes smaller as $p$ grows. In terms of the power of $m$, the upper bound in Theorem~\ref{thm:var-approx} improves upon the random-feature upper bound $O(m^{-1+1/p})$ obtained in \cite{chen2025dualityframeworkanalyzingrandom}.\footnote{Since no matching lower bound is known for our approximation result, this comparison should be understood as a comparison between available upper bounds rather than as a claim of optimality.} This is consistent with the model classes: a random feature model fixes the inner parameters, whereas a shallow network optimizes over them. At the same time, our proof invokes the Barron-space approximation theorem of \cite[Theorem~3]{siegel2025optimalapproximationzonoidsuniform} for the intermediate $\widetilde{\fcal}_{2,\tau_d,s}$ approximant, and the implicit constants may suffer from the curse of dimensionality. By contrast, the random-feature upper bound in \cite{chen2025dualityframeworkanalyzingrandom} has constants without such a dimensional curse.

The endpoint $p=2$ is not included in the theorem statement because it does not require the intermediate approximation step. In this case $f\in\widetilde{\fcal}_{2,\tau_d,s}$, which embeds continuously into the Barron space $\bscr_s$ under the above normalization. Therefore \cite[Theorem~3]{siegel2025optimalapproximationzonoidsuniform} gives
\[
    \inf_{f_m\in\Sigma_m^s}\|f_m-f\|_{L^2}
    \lesssim \|f\|_{\widetilde{\fcal}_{2,\tau_d,s}}m^{-\frac{2s+d+1}{2d}}\,.
\]

The proof of Theorem~\ref{thm:var-approx} is based on an intermediate approximation step. Instead of approximating $f\in\fcalwd$ directly by a finite network, we first approximate it by a function in the Hilbert-type space $\widetilde{\fcal}_{2,\tau_d,s}$ and then approximate this intermediate function by a shallow network. The first step is where the $L^p$ structure enters.

We begin on the sphere. After writing the parameter as $\thetab=(\wb^\top,b)\in\sbb^d$, the uniform measure $\tau_d$ allows us to diagonalize the integral operator associated with $\sigma_s(\thetab^\top\ub)$ by spherical harmonics. This gives the following approximation result.
\begin{lemma}[Approximation on the unit sphere]\label{lem:approxonsphere}
    Let $d\ge2$, $s\in\nbb_0$, $1\le p<2$, and $\tilde{f}\in\fcalwd(\sbb^d)$. Assume that $\tilde f$ has parity opposite to that of $s$.
    Then, for every $R>0$,
    \begin{align*}
        \inf_{\|g\|_{\widetilde{\fcal}_{2,\tau_d,s}(\sbb^d)} \leq R}\|\tilde{f}-g\|_{L^2(\sbb^d)}
        \lesssim\|\tilde f\|_{\fcalwd(\sbb^d)}
        \begin{cases}
            R^{-\frac{p(2s+2d+1)-2d}{d(2-p)}}\,,&1\le p\le p^*\,,\\
            R^{-\frac{p(4s+3d-1)-2d+2}{2d-2-pd+3p}}\,,&p^*<p<2\,,
        \end{cases}
    \end{align*}
    where $p^*=\frac{2d+2}{d+3}$.
\end{lemma}

For $p=2$, this intermediate approximation step is trivial: once $R\ge\|\tilde f\|_{\widetilde{\fcal}_{2,\tau_d,s}(\sbb^d)}$, one may take $g=\tilde f$, and the left-hand side is zero. The parity condition in Lemma~\ref{lem:approxonsphere} ensures that the filtered spherical-polynomial approximants remain in the admissible range of the ReLU$^s$ integral operator. The key point of the lemma is that, for $p<2$, the $L^p$ integrability of the outer coefficient $a(\thetab)$ can be converted into decay of the spherical-harmonic components of $\tilde f$. The resulting polynomial decay in $R$ improves as $p$ increases. The critical index $p^*$ marks the change in the spherical-harmonic projection estimate, not an exponential improvement of the intermediate approximation.

The next step transfers the spherical approximation back to the original domain. We use the lifting
\[
    \tilde f(\ub)=u_{d+1}^s f\left(\frac{\ub'}{u_{d+1}}\right),\qquad \ub=(\ub',u_{d+1})\in\sbb^d,\quad u_{d+1}>0\,,
\]
and then restrict the spherical approximant back to $\Omega$.
\begin{corollary}[From $\sbb^d$ to $\bbb^d$]\label{cor:fromspheretoball}
    Let $d\ge2$, $s\in\nbb_0$, $1\le p<2$, and $f\in\fcalwd$. Then, for every $R>0$,
    \begin{align*}
        \inf_{\|g\|_{\widetilde{\fcal}_{2,\tau_d,s}} \leq R}\|f-g\|_{L^2}
        \lesssim\|f\|_{\fcalwd}
        \begin{cases}
            R^{-\frac{p(2s+2d+1)-2d}{d(2-p)}}\,,&1\le p\le p^*\,,\\
            R^{-\frac{p(4s+3d-1)-2d+2}{2d-2-pd+3p}}\,,&p^*<p<2\,,
        \end{cases}
    \end{align*}
    where $p^*=\frac{2d+2}{d+3}$.
\end{corollary}

Combining Corollary~\ref{cor:fromspheretoball} with the Barron approximation bound for functions in $\widetilde{\fcal}_{2,\tau_d,s}$ gives the final network approximation. By the embedding used in the endpoint case and Lemma~\ref{thm:barron-approx}, which follows from \cite[Theorem~3]{siegel2025optimalapproximationzonoidsuniform}, if $\|g\|_{\widetilde{\fcal}_{2,\tau_d,s}}\le R$, then there exists $f_m\in\Sigma_m^s$ satisfying
\[
    \|g-f_m\|_{L^2}\lesssim Rm^{-\frac{2s+d+1}{2d}}\,.
\]
Balancing this term with the two polynomial regimes in Corollary~\ref{cor:fromspheretoball} gives the two cases of Theorem~\ref{thm:var-approx}. The proofs of Lemma~\ref{lem:approxonsphere}, Corollary~\ref{cor:fromspheretoball}, and the final balancing argument are given in Appendix~\ref{appendix:proof-var-approx}.

The proof of Theorem~\ref{thm:var-approx} uses the special structure of the uniform measure $\tau_d$ in Lemma~\ref{lem:approxonsphere}. For a general measure $\tau$, the Funk--Hecke formula is no longer available, so one cannot derive a spherical-harmonic expansion analogous to \eqref{eq:fourier_series}. Moreover, even when $\mathrm{d}\nu\ll\mathrm{d}\tau_d$, the Radon--Nikodym derivative need not inherit the integrability of $a(\wb,b)$, so the proof strategy breaks down. The critical value $p^*$ is tied to the boundedness of projection operators: the argument requires estimates for the modulus of smoothness on $L^p$-type spaces, and the $L^p$ boundedness of spherical-harmonic projection operators changes its form at this index \citep{kwon2018sharplplqestimatesspherical}. Thus $p^*$ should be understood as a breakpoint of the projection estimate; the intermediate approximation rate in Lemma~\ref{lem:approxonsphere} remains polynomial and continuous across this point.

\subsection{Approximation of Sobolev Spaces}

We next turn to Sobolev spaces. Existing shallow-network approximation results for Sobolev functions often pass through embeddings into Barron or spectral Barron spaces, but the available embeddings are most effective when the smoothness is measured in $L^2$ or stronger norms. For $1\le p<2$, such an embedding is more delicate. The following theorem gives an approximation bound by first embedding $W^{\alpha,p}$ into a spectral Barron space and then applying known spectral-Barron approximation estimates.
\begin{theorem}[Approximation of Sobolev spaces by shallow networks]\label{thm:sobolev-approx}
    Let $s\in\nbb$, $\alpha>0$, $1\le p<2$, and $f\in W^{\alpha,p}$. Then
    \begin{align*}
        \inf_{f_m\in\Sigma_{m}^s}\|f-f_m\|_{L^p}\lesssim\|f\|_{W^{\alpha,p}}\begin{cases}
            m^{-\frac{d+2s}{2(d+1)}}\,,&\alpha\ge s+d\,,\\
            m^{-\frac{\alpha(d+2s)}{2(s+d)(d+1)}}\,,&\alpha\in\nbb\cap\left(0,s+d\right)\,.
        \end{cases}
    \end{align*}
\end{theorem}

Theorem~\ref{thm:sobolev-approx} addresses the regime $1\le p<2$, which is not covered by \cite{MaoSiegelXu2026}. Its proof is based on a new embedding from the Sobolev space $W^{\alpha,p}$ into the spectral Barron space, which reduces the problem to the $L^2$ approximation theorem for spectral Barron spaces in \cite[Theorem~3]{SiegelXu2022HighOrder}. The resulting Sobolev rates are not expected to be optimal, because the embedding and smoothing steps may lose information specific to $W^{\alpha,p}$. The proof is given in Appendix~\ref{appendix:proof_sobolev}.

\section{Generalization Analysis}\label{sec:generalization}

The Barron norm is equivalent to the variation norm of infinite-width ReLU representations, while the $\ell_1$ path norm is its finite-width, parameter-level counterpart. This connection provides a natural way to control the complexity of trained networks through their parameters and motivates the path-norm regularization used below.
We consider empirical risk minimization with an $\ell_1$ path-regularization term:
\begin{align*}
    f_m\in\operatorname*{argmin}_{f_\theta\in\Sigma_{m}^s}\left\{\frac{1}{n}\sum_{i=1}^n|f_\theta(\xb_i)-y_i|^2+\lambda\|\theta\|_{\pscr_{1,s}}\right\}\,.
\end{align*}
For simplicity, we assume throughout that the above optimization problem admits a solution.

\subsection{Upper Bounds for the Generalization Error}
This subsection derives upper bounds for the generalization error of shallow ReLU$^s$ networks under suitable regularity and noise assumptions. The analysis uses the following two assumptions:
\begin{assumption}[Regularity assumption]\label{assumption:regularity}
    The regression function satisfies
    \[
        f_\rho\in W^{\alpha,\infty}(\Omega)\,,
        \qquad
        \|f_\rho\|_{W^{\alpha,\infty}(\Omega)}\le1\,,
    \]
    where $\alpha>0$ and, when $\alpha<s+\frac{d+1}{2}$, $\alpha$ is an integer.
\end{assumption}
\begin{assumption}[Noise assumption]\label{assumption:noise}
    Let ${\etab}=Y-f_\rho(X)$ denote the (zero-mean) noise. Then ${\etab}$ is sub-Gaussian, that is, there exists $\sigma>0$ such that
        \begin{align*}
            \ebb\left[e^{\lambda{\etab}}\right]\le e^{\frac{\lambda^2\sigma^2}{2}},\ \ \forall\lambda\in\rbb\,.
        \end{align*}
        Moreover, it is independent of $X$.
\end{assumption}

In Assumption~\ref{assumption:regularity}, the requirement $\alpha\in\nbb$ when $\alpha<s+\frac{d+1}{2}$ comes from the approximation theorem used later. Assumption~\ref{assumption:noise} is more general than the commonly used bounded-noise assumption: it includes bounded, Gaussian, and more general sub-Gaussian noise. An even broader condition is the sub-exponential assumption, namely, that there exist $R>0$ and $C>1$ such that
\begin{align*}
    \ebb[|{\etab}|^p|X=\xb]\le Cp!R^p,\ \ \forall p\in\nbb,x\in\Omega\,,
\end{align*}
for which analyses of unbounded errors can be found in \citep{Wainwright_2019,guozhengchu13,atail08}. Both the sub-Gaussian and sub-exponential conditions can be unified by requiring the Orlicz norm $\|\cdot\|_{\psi_p}$ to be finite. It is defined by
\begin{align*}
    \|X\|_{\psi_p}=\inf\{t>0:\ebb[\exp(|X|^p/t^p)]\le2\},\ \ p\ge1\,.
\end{align*}
When $p=1$ and $p=2$, this reduces to the sub-exponential and sub-Gaussian cases, respectively. For a truncation level $B>\|f_\rho\|_{L^\infty(\Omega)}$, let $\pi_B t=\max\{-B,\min\{t,B\}\}$ and define $\pi_B f=\pi_B\circ f$.\label{sym:first:pi_B_f} The main result of this subsection is the following theorem, proved in Appendix~\ref{appendix:proof-upper}.
\begin{theorem}\label{thm:upperbound}
    Under Assumptions~\ref{assumption:noise} and~\ref{assumption:regularity}, consider the empirical risk minimizer
    \begin{align*}
        f_m\in\arg\min_{f\in\Sigma_m^s}\left\{\frac{1}{n}\sum_{i=1}^n(f(\xb_i)-y_i)^2+\lambda\|f\|_{\pscr_{1,s}}\right\}\,.
    \end{align*}
    Then there exist constants $c_1,c_2>0$ such that:
    \begin{enumerate}
        \item If $\alpha\ge s+\frac{d+1}{2}$, then by taking
$$m\gtrsim n^{\frac{d}{2d+2s+1}},\ \ \lambda\asymp n^{-\frac{d+2s+1}{2d+2s+1}}\log n\label{eq:parameters1}
$$
\label{sym:first:gtrsim}\label{sym:first:asymp}
one has
        \begin{align*}
            \|\pi_B f_m-f_\rho\|_{L^2(\mathrm{d} \mu)}^2\lesssim n^{-\frac{d+2s+1}{2d+2s+1}}\log n\,,
        \end{align*}
        with probability at least $1-c_1e^{-c_2n^{\frac{d}{2d+2s+1}}\log n}$.
        \item If $\alpha\in\nbb\cap(0,s+\frac{d+1}{2})$, then by taking
$$m\gtrsim n^{\frac{d}{2\alpha+d}},\ \ \lambda\asymp n^{-\frac{1}{2}-\frac{2s+1}{4\alpha+2d}}\log n$$
one has
        \begin{align*}
            \|\pi_B f_m-f_\rho\|_{L^2(\mathrm{d} \mu)}^2\lesssim n^{-\frac{2\alpha}{2\alpha+d}}\log n\,,
        \end{align*}
        with probability at least $1-c_1e^{-c_2n^{\frac{d}{2\alpha+d}}\log n}$.
    \end{enumerate}
\end{theorem}

\subsection{Minimax Lower Bounds}
This subsection proves that the upper bound in Theorem~\ref{thm:upperbound} is minimax optimal. We continue to consider i.i.d. samples $\{(\xb_i, y_i)\}_{i=1}^n$ drawn from an unknown distribution $\mathrm{d}\rho$, and estimators of the form
\begin{align*}
    f_m \in \arg\min_{f \in \Sigma_m^s} \left\{ \frac{1}{n} \sum_{i=1}^n (f(\xb_i) - y_i)^2 + \lambda \|f\|_{\pscr_{1,s}} \right\}\,.
\end{align*}
For the lower-bound argument, we impose the following two assumptions:
\begin{assumption}[Noise assumption]\label{assumption:noise2}
    The samples satisfy
    \begin{align*}
        y_i = f_\rho(\xb_i) + \eta_i,\quad \eta_i \sim \ncal(0, \sigma^2)\,.
    \end{align*}
    Here $\sigma^2 > 0$ is the variance. 
\end{assumption}
\begin{assumption}[Sample distribution assumption]\label{assumption:sample}
    Assume that the marginal distribution of $X$ is the uniform distribution on $\Omega$, namely, $\mu = U(\Omega)$. 
\end{assumption}

Assumption~\ref{assumption:noise2} is a special case of Assumption~\ref{assumption:noise}. We impose Gaussian noise mainly because it simplifies the computation of the Kullback--Leibler divergence. Similar arguments should also apply to certain other sub-Gaussian distributions, such as log-concave or light-tailed laws. Assumption~\ref{assumption:sample} is introduced because we need to estimate the metric entropy $\log \ncal(\varepsilon, W^{\alpha,\infty}(\Omega), \|\cdot\|_{L^2(\mathrm{d} \mu)})$. If $\mu = U(\Omega)$, then
\begin{align*}
    \log \ncal(\delta, B W^{\alpha,\infty}(\Omega), \|\cdot\|_{L^2(\mathrm{d} \mu)}) \asymp \log \ncal(\delta, B W^{\alpha,\infty}(\Omega), \|\cdot\|_{L^2(\Omega)})\,.
\end{align*}
In fact, it is enough to assume that $\mu \ll m$ and that $w = \frac{\mathrm{d} \mu}{dx}$ is bounded, where $m$ denotes Lebesgue measure; under this condition, the same conclusion continues to hold.

Given a target function class $\fcal$, we define the \emph{minimax risk} by
\begin{align*}
    \inf_{\hat{f}} \sup_{f_\rho \in \fcal} \ebb \left[ \| \hat{f} - f_\rho \|_{L^2(\mathrm{d} \mu)}^2 \right]\,,
\end{align*}
where the infimum is taken over all estimators, that is, all measurable functions of the sample $\{(\xb_i, y_i)\}_{i=1}^n$. We are interested in lower bounds of order $n^{-\beta}$. If such a lower bound matches the upper bound in Theorem~\ref{thm:upperbound}, then the estimator $f_m$ is statistically optimal.

\begin{theorem}\label{thm:lowerbound1}
    Let $\alpha > 0$. Then under Assumptions~\ref{assumption:noise2} and~\ref{assumption:sample}, the minimax risk satisfies
    \begin{align*}
        \inf_{\hat{f}} \sup_{\substack{f_\rho \in W^{\alpha,\infty}(\Omega), \\ \|f_\rho\|_{W^{\alpha,\infty}(\Omega)} \le 1}} \ebb \left[ \| \hat{f} - f_\rho \|_{L^2(\mathrm{d} \mu)}^2 \right] \gtrsim n^{-\frac{2\alpha}{2\alpha + d}}\,.
    \end{align*}
\end{theorem}

It follows that when $\alpha < s + \frac{d+1}{2}$, the upper bound in Theorem~\ref{thm:upperbound} matches the minimax lower bound up to logarithmic factors, so the clipped estimator $\pi_B f_m$ is statistically optimal. By contrast, when $\alpha \ge s + \frac{d+1}{2}$, the upper bound no longer matches the Sobolev minimax lower bound. The reason is that our upper-bound argument uses the embedding $W^{\alpha,\infty}(\Omega)\subset \bscr_s(\Omega)$, which effectively reduces the problem to the Barron class and yields the rate $n^{-\frac{d+2s+1}{2d+2s+1}}\log n$. The next theorem shows that this rate is itself minimax optimal over $\bscr_s(\Omega)$.

\begin{theorem}\label{thm:lowerbound2}
    Under Assumptions~\ref{assumption:noise2} and~\ref{assumption:sample}, the minimax risk satisfies
    \begin{align*}
        \inf_{\hat{f}} \sup_{\substack{f_\rho \in \bscr_s(\Omega), \\ \|f_\rho\|_{\bscr_s(\Omega)} \le 1}} \ebb \left[ \| \hat{f} - f_\rho \|_{L^2(\mathrm{d} \mu)}^2 \right] \gtrsim n^{-\frac{d+2s+1}{2d+2s+1}}\,.
    \end{align*}
\end{theorem}

\section{Conclusion and Discussion}\label{sec:conclusion}
This paper developed approximation and generalization results for shallow ReLU$^s$ networks. For the $L^p$-type space $\fcalwd$, spherical harmonic analysis gives approximation rates for $1\le p<2$ (Theorem~\ref{thm:var-approx}), with the endpoint $p=2$ handled through the Barron-space embedding. The breakpoint $p^*=\frac{2d+2}{d+3}$ reflects a change in the relevant spherical-harmonic projection estimate. For Sobolev spaces $W^{\alpha,p}$ with $1\le p<2$, embeddings into spectral Barron spaces yield $L^p$ approximation bounds (Theorem~\ref{thm:sobolev-approx}) in a regime not covered by \cite{MaoSiegelXu2026}.

On the statistical side, generic chaining gives generalization bounds for $\ell_1$ path-norm-regularized shallow ReLU$^s$ networks under sub-Gaussian noise. The resulting rates are
\[
    O\!\left(n^{-\frac{d+2s+1}{2d+2s+1}}\log n\right)
    \quad\text{over }\bscr_s,\qquad
    O\!\left(n^{-\frac{2\alpha}{2\alpha+d}}\log n\right)
    \quad\text{over }W^{\alpha,\infty}\,.
\]
Together with the minimax lower bounds in Theorems~\ref{thm:lowerbound1} and~\ref{thm:lowerbound2}, these upper bounds are optimal up to logarithmic factors in their respective statistical settings.

Several questions remain.

\begin{itemize}
    \item \textbf{Optimality and dimension dependence for $L^p$-type approximation.} Theorem~\ref{thm:var-approx} gives upper bounds for shallow-network approximation over $L^p$-type spaces. Matching lower bounds for these approximation rates are not known. In addition, our proof uses the Barron-space approximation theorem for the intermediate $\widetilde{\fcal}_{2,\tau_d,s}$ approximant, so the implicit constants may suffer from the curse of dimensionality. It would be useful to determine whether these rates are sharp and whether the dimension dependence in the constants can be reduced.

    \item \textbf{Approximation under general parameter measures.} The approximation results for $L^p$-type spaces rely on taking the parameter measure to be the uniform distribution $\tau_d$, which makes the Funk--Hecke formula available. For general probability measures, the present proof no longer applies. Extending the theory beyond the uniform spherical measure remains an open problem.

    \item \textbf{Direct Sobolev approximation for $1\le p<2$.} The rates in Theorem~\ref{thm:sobolev-approx} are not known to be optimal. Although the spectral-Barron approximation theorem used in the proof is sharp in its own setting, the embedding and smoothing steps may lose information specific to Sobolev spaces. A direct analysis of $W^{\alpha,p}$, possibly using Littlewood--Paley decompositions or related harmonic-analysis tools, may lead to sharper bounds.
\end{itemize}

These questions point to a broader problem: to understand how representation norms, approximation mechanisms, and statistical complexity interact beyond the regimes where Barron-type embeddings are already well developed.


\acks{The work of Lei Shi is partially supported by the National Natural Science Foundation of China [Grant No.12571099].}


\newpage

\appendix
\section{Symbol Table}\label{appendix:symbols}
\small
\newcommand{\symnopage}{\sout{--}}
\begin{center}
\begin{longtable}{p{3.8cm} p{7.7cm} p{2cm}}
\hline
Symbol & Meaning & Def. page \\
\hline
$A\lesssim B$ & There exists a constant $C>0$ such that $A\le CB$ & \symnopage \\
$A\gtrsim B$ & There exists a constant $C>0$ such that $A\ge C^{-1}B$ & \symnopage \\
$A\asymp B$ & Both $A\lesssim B$ and $A\gtrsim B$ hold & \symnopage \\
$O(\cdot)$ & Big O notation & \symnopage \\
$L^p(d\nu)$\label{sym:row:Lpdnu} & $L^p$ space with respect to the measure $d\nu$ & \symnopage \\
$\|\cdot\|_{L^p(d\mu)}$\label{sym:row:Lpdmu} & Norm on the space $L^p(d\mu)$ & \symnopage \\
$\fcal(\Omega)$\label{sym:row:fcalOmega} & Function space $\fcal$ with domain $\Omega$ explicitly indicated & \symnopage \\
$\fcal(R)$\label{sym:row:fcalR} & Unit ball of radius $R$ in the function space $\fcal$ & \symnopage \\
$\bbb^{d}$ & $\{x\in\rbb^d:\|\xb\|_2\le1\}$ unit ball & \symnopage \\
$\sbb^{d-1}$ & $\{\xb\in\rbb^d:\|\xb\|_2=1\}$ unit sphere & \symnopage \\
$\pi_B f$ & Truncated function & \pageref{sym:first:pi_B_f} \\
$\pi_B(\fcal)$ & Family of truncated functions & \pageref{sym:first:pi_B_fcal} \\
$\Sigma_m^s$ & Class of shallow networks with $m$ neurons and ReLU$^s$ activation & \pageref{sym:first:Sigma_ms} \\
$\bscr_s$ & Barron space of order $s$ & \pageref{sym:first:bscr_s} \\
$\|\cdot\|_{\bscr_s}$ & Barron norm of order $s$ & \pageref{sym:first:bscr_norm} \\
$\fcalw$ & Variation space & 6 \\
$\|\cdot\|_{\fcalw}$ & Variation space norm & \pageref{sym:first:fcalw_norm} \\
$W^{s,p}(\Omega)$ & Sobolev space of order $s$ & \pageref{sym:first:Sobolev} \\
$\|\cdot\|_{W^{s,p}}$ & Sobolev norm & \pageref{sym:first:Sobolev_norm} \\
$\ncal(\varepsilon,\fcal,\|\cdot\|)$ & $\varepsilon$-covering number & \pageref{sym:first:covering_number} \\
$\gcal_n(\fcal,\delta,\xib)$ & Local complexity & \pageref{sym:first:gcal_n} \\
$\starhull(\fcal)$ & Star-shaped hull & \pageref{sym:first:starhull} \\
$\gamma_2(T,d)$ & Generic chaining complexity & \pageref{sym:first:gamma_2} \\
$\sigma_s(\cdot)$ & ReLU$^s$ activation function & \pageref{sym:first:sigma_s} \\
$\pscr_{1,s}$ & Path norm of order $s$ & \pageref{sym:first:pscr_s} \\
$\tau_d$ & Uniform measure on $\sbb^d$ & \pageref{sym:first:tau_d} \\
$\omega_d$ & Surface area of $\sbb^d$ & \pageref{sym:first:omega_d} \\
$\chi_A$\label{sym:row:chiA} & Indicator function of set $A$ & \symnopage \\
$\delta_{\xb}$ & Dirac measure at $\xb$ & \symnopage \\
$\ncal(\cdot)$ & Covering number & \symnopage \\
$\ncal_{KL}(\varepsilon;\pcal)$\label{sym:row:ncalKL} & $\varepsilon$-covering number under KL divergence & \symnopage \\
$\fscr_s$ & Spectral Barron space of order $s$ & \pageref{sym:first:fscr_s} \\
$\|\cdot\|_{\fscr_s}$ & Spectral Barron norm of order $s$ & \pageref{sym:first:fscr_norm} \\
$\tscr(\rbb^d)$ & Schwartz function space & \pageref{sym:first:tscr} \\
$\tscr'(\rbb^d)$ & Space of tempered distributions & \pageref{sym:first:tscr_prime} \\
$\pbb_s$ & Dictionary of ReLU$^s$ activation functions & \pageref{sym:first:pbb_s} \\
$\kscr(\dbb)$ & Variation space generated by the dictionary $\dbb$ & \pageref{sym:first:kscr_dbb} \\
$\|\cdot\|_{\dbb}$ & Variation norm induced by the dictionary $\dbb$ & \pageref{sym:first:dbb_norm} \\
$\mathrm{KL}(P\|\|Q)$ & Kullback-Leibler divergence & \symnopage \\
$I(X;Y)$ & Mutual information & \symnopage \\
$\sim$\label{sym:row:sim} & Asymptotically equivalent / identically distributed & \symnopage \\
$\hookrightarrow$ & Embedding of spaces & \symnopage \\
\hline
\end{longtable}
\end{center}

\section{Proof of Theorem \ref{thm:barron-approx}}\label{appendix:proof-barron-approx}
We prove Theorem \ref{thm:barron-approx} through the variation space defined by a Minkowski functional \citep{Siegel_2022sharp,siegel2022characterizationvariationspacescorresponding}. Let $X$ be a separable Banach space and let $\dbb\subset X$ be a uniformly bounded set, often referred to as a \emph{dictionary}. Consider the closed symmetric convex hull of $\dbb$ in $X$:
\begin{align*}
    \dscc=\overline{\left\{\sum_{j=1}^na_jh_j:n\in\nbb,h_j\in\dbb,\sum_{i=1}^n|a_i|\le1\right\}}
\end{align*}
and the Minkowski functional induced by it:
\begin{align*}
    \|f\|_{\dbb}=\inf\left\{c>0:\frac{f}{c}\in \dscc\right\}\,,
\end{align*}
\label{sym:first:dbb_norm}
which we call the \emph{variation norm} induced by $\dbb$. We call
\begin{align*}
    \kscr(\dbb)=\left\{f\in X:\|f\|_{\dbb}<\infty\right\}
\end{align*}
\label{sym:first:kscr_dbb}
the \emph{variation space} generated by $\dbb$.

Consider the dictionary
\begin{align*}
    \pbb_s=\left\{\sigma_s\left(\wb^\top\xb+b\right):\wb \in\sbb^{d-1},b\in [-1,1]\right\}\subset L^2\,,
\end{align*}
\label{sym:first:pbb_s}
The set $\pbb_s$ is uniformly bounded, and as $(\wb ,b)$ traverses $\sbb^{d-1}\times [-1,1]$, the hyperplane $\wb^\top\xb+b$ sweeps across $\Omega $. Consequently, polynomials of degree at most $s$ on $\wb $ can be expressed as linear combinations of $\pbb_s$ \citep{siegel2022characterizationvariationspacescorresponding}. Functions in the variation space can be written in integral form, as stated in the following proposition.
\begin{lemma}[\cite{Yang2024}]
    Let $s\in\nbb_0$. Then
    \begin{align*}
        \overline{\mathrm{conv}(\pm\pbb_s)}=\left\{f(x)=\int_{\sbb^{d-1}\times [-1,1]}\sigma_s\left(\wb^\top\xb+b\right)d\mu(\wb ,b):\|\mu\|\le1\right\}\,,
    \end{align*}
    where $\mu$ is a signed Radon measure on $\sbb^{d-1}\times [-1,1]$, and
    \begin{align}\label{eq:interform}
        \|f\|_{\pbb_s}=\inf\left\{\|\mu\|:f(x)=\int_{\sbb^{d-1}\times [-1,1]}\sigma_s\left(\wb^\top\xb+b\right)d\mu(\wb ,b)\right\}\,.
    \end{align}
\end{lemma}

This integral representation space has been studied in many works \citep{Yang2024,chen2025dualityframeworkanalyzingrandom,bach2016breakingcursedimensionalityconvex}. On the compact set $\Omega=\bbb^d$ (similar results hold for compact sets satisfying $\mathrm{Span}\{\Omega\}=\rbb^d$), the variation space $\kscr(\pbb_s)$ and the Barron space $\bscr_s$ are essentially equivalent. \cite{siegel2022characterizationvariationspacescorresponding} proved this for the case $s=1$; we extend their result to $s\in\nbb_0$.
\begin{lemma}\label{lem:kpsequivbs}
    Let $s\in\nbb_0$. Then $\kscr(\pbb_s)=\bscr_s$ and
    \begin{align*}
        \|f\|_{\bscr_s}\lesssim_{d,s}\|f\|_{\pbb_s}\lesssim_{s}\|f\|_{\bscr_s}\,.
    \end{align*}
\end{lemma}
\begin{proof}
    We prove only the case $s\ge1$; the proof for $s=0$ is entirely similar. Consider the dictionary
    \begin{align*}
        \bbb=\left\{(\|\wb \|_1+|b|)^{-s}\sigma_s\left(\wb^\top\xb+b\right):(\wb,b)\in\rbb^{d+1}\setminus\{0\}\right\}\subset L^2\,.
    \end{align*}
    The dictionary $\bbb$ is uniformly bounded and compact in $L^2$. For any $f\in\kscr(\bbb)$,
    \begin{align*}
        f(x)=\int_{\bbb}(\|\wb \|_1 +|b|)^{-s}\sigma_s\left(\wb^\top\xb+b\right)d\mu\,,
    \end{align*}
    where $\mu$ is a signed Radon measure on $\bbb$. Perform a Jordan decomposition $\mu=\mu^+-\mu^-$ and let $\nu=\mu^++\mu^-$. Let $\varphi(\wb,b)=\frac{d\mu}{d\nu}(\wb,b)$ be the Radon--Nikodym derivative of $\mu$ with respect to $\nu$. Define $a=a(\wb,b)=(\|\wb \|_1+|b|)^{-s}\varphi(\wb,b)\|\nu\|$ and make the change of variable $\rho=\frac{\nu}{\|\nu\|}$ to obtain
    \begin{align*}
        f(x)=\ebb_{(a,\wb ,b)\sim\rho}[a\sigma_s\left(\wb^\top\xb+b\right)],\kongge x\in\Omega
    \end{align*}
    and
    \begin{align*}
        \|\mu\|=\ebb_{(a,\wb ,b)\sim\rho}[|a|(\|\wb \|_1 +|b|)^s]\,.
    \end{align*}
    Conversely, for any $f\in\bscr_s$, one can also find a measure $\mu$ on $\bbb$ such that $f=\int_{\bbb}i_{\bbb\to L^2}d\mu$ and the above equation holds. Thus, $\kscr(\bbb)=\bscr_s$ and $\|f\|_{\kscr(\bbb)}=\|f\|_{\bscr_s}$.

    We prove the lower bound first. It is enough to show
    \[
        \pbb_s\subset C\overline{\mathrm{conv}(\pm\bbb)}\,.
    \]
    Indeed, for any $g(x)=\sigma_s\left(\wb^\top\xb+b\right)\in\pbb_s$, where $\wb \in\sbb^{d-1}$ and $b\in[-1,1]$, since
    \begin{align*}
        (\|\wb \|_1+|b|)^s\le(\sqrt d+1)^s\lesssim_s d^{s/2}\,,
    \end{align*}
    and $(\|\wb \|_1+|b|)^{-s}\sigma_s\left(\wb^\top\xb+b\right)\in\bbb$, it follows that $g\in Cd^{\frac{s}{2}}\overline{\mathrm{conv}(\pm\bbb)}$. Consequently, $\overline{\mathrm{conv}(\pm \pbb_s)}\subset Cd^{\frac{s}{2}}\overline{\mathrm{conv}(\pm\bbb)}$. For any $c>0$ satisfying $\frac{f}{c}\in\overline{\mathrm{conv}(\pm\pbb_s)}$, we have $f\in cCd^{\frac{s}{2}}\overline{\mathrm{conv}(\pm\bbb)}$, and thus
    \begin{align*}
        \|f\|_{\bscr_s }=\|f\|_{\bbb}\le cCd^{\frac{s}{2}}\,.
    \end{align*}
    By the arbitrariness of $c$, the first inequality follows.

    For the upper bound, a similar argument reduces the proof to showing $\bbb\subset C\pscc$. For any $g=(\|\wb \|_1+|b|)^{-s}\sigma_s\left(\wb^\top\xb+b\right)\in\bbb$, where $\wb \in\rbb^d$ and $b\in\rbb$, suppose first that $\|\wb \|_1\neq0$. By the positive homogeneity of degree $s$ of $\sigma_s$, we may normalize $\wb$ so that $\wb\in\sbb^{d-1}$. Since $\|\wb\|_1\ge1$, the coefficient $(\|\wb\|_1+|b|)^{-s}$ is no larger than $(1+|b|)^{-s}$. As $\pscc$ is symmetric and convex, it suffices to prove
    \begin{align*}
        \tilde{g}(x)=(1+|b|)^{-s}\sigma_s\left(\wb^\top\xb+b\right)\in C\pscc\,.
    \end{align*}
    If $b\in[-1,1]$, then taking $C=1$ suffices. If $b<-1$, then $\tilde{g}\equiv0\in\pscc$. If $b>1$, then
    \begin{align*}
        \tilde{g}(x)=(1+|b|)^{-s}(\wb^\top\xb+b)^s=(1+|b|)^{-s}\sum_{i=0}^s{\binom{s}{i}b^{s-i}}(\wb^\top\xb)^i\,.
    \end{align*}
    By \cite[discussion after Proposition~6]{bach2016breakingcursedimensionalityconvex}, there exists a constant $C=C(d,s)$ such that for $i=0,\cdots,s$, $(\wb^\top\xb)^i\in C\pscc$. Moreover,
    \begin{align*}
        (1+|b|)^{-s}\sum_{i=0}^s\binom{s}{i}|b|^{s-i}=1\,,
    \end{align*}
    hence $\tilde{g}\in C(d,s)\pscc$. If $\|\wb \|_1=0$, then $g(x)$ is identically $1$ or $0$ on $\Omega$, and thus also belongs to $C\pscc$. This completes the proof of Lemma~\ref{lem:kpsequivbs}.
\end{proof}

\par\noindent{\bf Proof of Theorem~\ref{thm:barron-approx}.}
This is a direct corollary of \citep[Theorem 3]{siegel2025optimalapproximationzonoidsuniform} and Lemma~\ref{lem:kpsequivbs}. This completes the proof of Theorem~\ref{thm:barron-approx}.
\hfill\BlackBox\\[2mm]

\section{Proof of Theorem \ref{thm:var-approx}}\label{appendix:proof-var-approx}
This appendix follows the approach of \citep{bach2016breakingcursedimensionalityconvex,Yang2024} to prove Theorem \ref{thm:var-approx}. We first transform the problem into an approximation problem on the sphere. To use spherical harmonic analysis, we take the parameter space in this section to be $\sbb^d$, instead of $\sbb^{d-1}\times I$, and write the parameter as $\thetab=(\wb^\top,b)\in\sbb^d$. We begin with a brief review of harmonic analysis on the sphere \citep{2013Approximation}. For $i\in\nbb_0$, the space of spherical harmonics of degree $i$, $\ybb_{i}$, is the linear space of restrictions to $\sbb^d$ of real homogeneous harmonic polynomials of degree $i$ on $\rbb^{d+1}$. The dimension of $\ybb_{i}$ is
\begin{align*}
    N(d,i):=\begin{cases}
    \frac{2i+d-1}{i}\binom{i+d-2}{d-1} \,,&\ i\neq 0 \\
    1 \,,&  \ i=0\,.
    \end{cases}
\end{align*}
Spherical harmonics are eigenfunctions of the Laplace-Beltrami operator $\Delta$:
\begin{align}\label{eq:eigenvalue}
\Delta Y_{i}=-i(i+d-1)Y_{i},\quad Y_{i}\in\ybb_{i}\,,
\end{align}
where, in coordinates $\ub=(u_{1},\ldots,u_{d+1})\in\sbb^d$,
\begin{align*}
\Delta=\sum_{k=1}^{d+1}\frac{\partial^{2}}{\partial u_{k}^{2}}-\sum_{k=1}^{d+1}\sum_{\ell=1}^{d+1}u_{k}u_{\ell}\frac{\partial^{2}}{\partial u_{k}\partial u_{\ell}}-d\sum_{k=1}^{d+1}u_{k}\frac{\partial}{\partial u_{k}}\,.
\end{align*}
Define the inner product $\langle g_1,g_2\rangle=\int_{\sbb^d}g_1(\ub)g_2(\ub)d\tau_d(\ub)$\label{sym:first:tau_d} on $\sbb^d$; then spherical harmonics of different degrees are orthogonal with respect to this inner product.

Any $g\in L^2(\sbb^d)$ has a unique decomposition in $L^2(\sbb^d)$:
\begin{align*}
    g(\ub)=\int_{\sbb^d}g(\vb)d\tau_d(\vb)+\sum_{i=1}^\infty N(d,i)\int_{\sbb^d}g(\vb)P_i(\ub^\top\vb)d\tau_d(\vb)=:\sum_{i=0}^\infty g_i(\ub)\,,
\end{align*}
where $P_{i}$ are the Gegenbauer polynomials:
\begin{align*}
P_{i}(t):=\frac{(-1)^{i}}{2^{i}}\frac{\Gamma(d/2)}{\Gamma(i+d/2)}(1-t^{2})^{(2-d)/2}\left(\frac{d}{dt}\right)^{i}(1-t^{2})^{i+(d-2)/2},\quad t\in[-1,1]\,,
\end{align*}
which satisfy $\|P_i\|_{L^\infty[-1,1]}\le1$ and $1-P_i(\cos\beta)\asymp\min\{1,i^2\beta^2\}$. The above decomposition can be viewed as an extension of Fourier series to the sphere. For $g\in L^2(\sbb^d)$, denote by $g_i$ its projection onto the spherical harmonic space of degree $i$. We also have the Parseval identity on the sphere:
\begin{align}\label{eq:parseval}
    \|g\|_{L^2(\sbb^d)}^2=\sum_{i=0}^\infty \|g_i\|_{L^2(\sbb^d)}^2\,.
\end{align}

A useful observation is that if $g\in\fcalwd$ and
\[
    g(\ub)=\int_{\sbb^d}a(\thetab)\sigma_s(\thetab^\top \ub)d\tau_d(\thetab)\,,
\]
then the Funk--Hecke formula \citep{bach2016breakingcursedimensionalityconvex} gives the Fourier series decomposition
\begin{align}\label{eq:fourier_series}
    g(\ub)=\sum_{i=0}^\infty \lambda_ia_i(\ub)\,,
\end{align}
where $a_i$ is the projection of $a$ onto $\ybb_i$, and $\lambda_i$ is the Gegenbauer coefficient of the kernel $\sigma_s$:
\begin{align*}
\lambda_i:=\frac{\omega_{d-1}}{\omega_{d}}\int_{-1}^{1}\sigma_s(t)P_{i}(t)(1-t^{2})^{\frac{d-2}{2}}dt\,.
\end{align*}
\label{sym:first:omega_d}
\cite{bach2016breakingcursedimensionalityconvex} provided the following calculations and estimates for these coefficients:
\begin{align*}
    \lambda_0=\frac{\omega_{d-1}}{\omega_d}\frac{\Gamma\left(\frac{d}{2}\right)\Gamma\left(\frac{s+1}{2}\right)}{2\Gamma\left(\frac{s+d+1}{2}\right)}\,,
\end{align*}
and, for $i\ge s+1$,
\begin{align*}
    \lambda_i=\begin{cases}
        0\,,&i\equiv s\pmod 2,i\ge s+1\\
        \frac{\Gamma\left(\frac{d}{2}\right)\Gamma(i-s)}{\Gamma\left(\frac{i-s+1}{2}\right)\Gamma\left(\frac{i+d+s+1}{2}\right)}\,,&i\equiv s+1\pmod 2,i\ge s+1
    \end{cases}\,.
\end{align*}
The remaining finitely many coefficients with $1\le i\le s$ are fixed constants depending only on $d$ and $s$ and do not affect the asymptotic estimates below.
By Stirling's formula, when $\lambda_i\neq0$, we have the asymptotic estimate
\begin{align}\label{eq:lambdaasymp}
    \lambda_i\asymp i^{-\frac{d+2s+1}{2}}\,.
\end{align}
In particular, when $p=2$, by the Parseval identity \eqref{eq:parseval}, we can express the $\widetilde{\fcal}_{2,\tau_d,s}$ norm of $g$ as
\begin{align*}
    \|g\|_{\widetilde{\fcal}_{2,\tau_d,s}}^2=\sum_{\lambda_i\neq0} \frac{1}{\lambda_i^2}\|g_i\|_{L^2(\sbb^d)}^2\,.
\end{align*}
It should be noted that for a general measure $\nu\neq\tau_d$, functions in $\fcalw(\sbb^d)$ generally do not admit such a simple Fourier series decomposition and norm expression; therefore, we choose $\nu=\tau_d$ in our analysis. For approximation theorems under more general probability measures, more refined analytical tools may be required.

Next, we introduce the smoothness of functions on the sphere. For $0\leq\beta\leq\pi$, the translation operator $T_{\beta}$ (spherical mean operator) is defined as
\begin{align*}
T_{\beta}g(\ub):=\int_{\sbb^{\bot}_{\ub}}g(\ub\cos\beta+\vb\sin\beta)d\tau_{d-1}(\vb),\quad \ub\in\sbb^d,g\in L^{1}(\sbb^d)\,,
\end{align*}
where $\sbb^{\bot}_{\ub}:=\{\vb\in\sbb^d:\ub^\top\vb=0\}$ is the equator of $\sbb^d$ relative to $\ub$ (isomorphic to $\sbb^{d-1}$). For $\alpha>0$ and $0<\beta<\pi$, we define the $\alpha$-th order difference operator
\begin{align*}
\Delta^{\alpha}_{\beta}:=(I-T_{\beta})^{\alpha/2}=\sum_{j=0}^{\infty}(-1)^{j}\binom{\frac{\alpha}{2}}{j}T_{\beta}^{j}\,,
\end{align*}
where $\binom{\alpha/2}{j}=\frac{(\alpha/2)(\alpha/2-1)\cdots(\alpha/2-j+1)}{j!}$, and it satisfies
\begin{align}\label{eq:moduliofsmoothness}
    (\Delta^{\alpha}_{\beta}g)_i=(1-P_{i}(\cos\beta))^{\frac{\alpha}{2}}g_i,\ i\in\nbb_0\,.
\end{align}
For $g\in L^{p}(\sbb^d)$ with $1\leq p\le2$, its $\alpha$-th order modulus of smoothness is defined as
\begin{align*}
\omega_{\alpha}(g,t)_{p}:=\sup_{0<\beta\leq t}\|\Delta^{\alpha}_{\beta}g\|_{L^{p}(\sbb^d)},\quad 0<t<\pi\,.
\end{align*}

With the above preparations, we now prove the intermediate approximation results stated in the main text.
\par\noindent{\bf Proof of Lemma~\ref{lem:approxonsphere}.}
    By homogeneity, assume $\|\tilde f\|_{\fcalwd(\sbb^d)}\le1$. Write
    \[
        \tilde{f}(\ub) = \int_{\sbb^d} a(\thetab) \sigma_s(\thetab^\top \ub) \, d\tau_d(\thetab)\,,
        \qquad \|a\|_{L^p(\sbb^d)} \leq 1\,.
    \]
    We first estimate the modulus of smoothness of $\tilde{f}$. From the previous discussion, apart from finitely many low-degree modes, the nonzero multipliers $\lambda_i$ satisfy $i \ge s+1$, $i+1 \equiv s \pmod{2}$, and $\lambda_i \asymp i^{-\frac{d+2s+1}{2}}$. Denote this high-frequency index set by $\mathcal{I}_s$. By the parity assumption in Lemma~\ref{lem:approxonsphere}, $\tilde f_i=0$ whenever $i\equiv s\pmod 2$, so no high-frequency inadmissible modes occur. The finitely many low-degree modes contribute only a constant multiple of $\beta^{2\alpha}$ to the estimate below, so they can be absorbed into the same bound. By the Parseval identity \eqref{eq:parseval} and \eqref{eq:moduliofsmoothness}, we have
    \[
        \|\Delta_\beta^\alpha \tilde{f}\|_{L^2(\sbb^d)}^2
        \lesssim \beta^{2\alpha}+\sum_{i \in \mathcal{I}_s} \lambda_i^2 (1 - P_i(\cos \beta))^\alpha \|a_i\|_{L^2(\sbb^d)}^2\,.
    \]
    According to the spherical harmonic projection theorem \cite[Theorem 9.1.1]{2013Approximation}, applied with ambient dimension $d+1$ and $\kappa=0$,
    \[
        \|a_i\|_{L^2(\sbb^d)} \lesssim i^{\zeta(p)}\,,
    \]
    where
    \begin{align*}
        \zeta(p)=\begin{cases}
            d\left(\frac1p-\frac12\right)-\frac12\,,&1\le p\le\frac{2d+2}{d+3}\,,\\
            \frac{d-1}{2}\left(\frac1p-\frac12\right)\,,&\frac{2d+2}{d+3}<p\le2\,.
        \end{cases}
    \end{align*}
    Set $\gamma:=2\zeta(p)-(d+2s+1)$. Substituting the asymptotic estimate for $\lambda_i$ from \eqref{eq:lambdaasymp} and using $(1 - P_i(\cos\beta))^\alpha \lesssim \min(1, i^{2\alpha} \beta^{2\alpha})$, we obtain
    \begin{align*}
        \|\Delta_\beta^\alpha \tilde{f}\|_{L^2}^2
        \lesssim& \beta^{2\alpha}+\sum_{i \in \mathcal{I}_s} i^{\gamma} \min(1, i^{2\alpha} \beta^{2\alpha})\\
        \lesssim& \sum_{k=1}^\infty k^{\gamma} \min(1, k^{2\alpha} \beta^{2\alpha})\,.
    \end{align*}
    Let $K_\beta = \left\lfloor \frac{1}{\beta} \right\rfloor$ and split the sum into two parts: $k \le K_\beta$ and $k > K_\beta$:
    \[
        S_1 = \beta^{2\alpha} \sum_{k=1}^{K_\beta} k^{\gamma + 2\alpha}, \quad S_2 = \sum_{k=K_\beta+1}^{\infty} k^{\gamma}\,.
    \]
    Note that the definition of $\zeta(p)$ gives $\gamma<-1$, so the series $S_2$ always converges. If $\gamma + 2\alpha > -1$, then $S_1 \asymp \beta^{2\alpha} K_\beta^{\gamma+2\alpha+1} \asymp \beta^{-\gamma-1}$, and $S_2 \asymp K_\beta^{\gamma+1} \asymp \beta^{-\gamma-1}$. If $\gamma + 2\alpha = -1$, then $S_1 \asymp \beta^{2\alpha} \log K_\beta \asymp \beta^{-\gamma-1} \log\frac{1}{\beta}$, and $S_2 \asymp \beta^{-\gamma-1}$. If $\gamma + 2\alpha < -1$, then $S_1 \asymp \beta^{2\alpha}$, and $S_2 \asymp \beta^{-\gamma-1}$. In this case, since $\gamma < -1$ and $2\alpha < -\gamma-1$, the term $\beta^{2\alpha}$ dominates. Combining the above, we get
    \[
        \|\Delta_\beta^\alpha \tilde{f}\|_{L^2}^2 \lesssim 
        \begin{cases}
            \beta^{2\alpha}\,, &  \gamma + 2\alpha < -1\,,\\
            \beta^{-\gamma-1} \log\frac{1}{\beta}\,, &  \gamma + 2\alpha = -1\,,\\
            \beta^{-\gamma-1}\,, &  \gamma + 2\alpha > -1\,.
        \end{cases}
    \]
    For $0 < t \le \pi$, since $\omega_\alpha(\tilde{f}, t)_2 = \sup_{0<\beta\le t} \|\Delta_\beta^\alpha \tilde{f}\|_{L^2}$ and the estimate is monotonic, we have
    \begin{align}\label{eq:modulibound}
        \omega_\alpha(\tilde{f}, t)_2 \lesssim 
        \begin{cases}
            t^{\alpha}\,, &  \alpha < \gamma^*\\
            t^{\gamma^*} \sqrt{\log\frac{1}{t}}\,, &  \alpha = \gamma^*\\
            t^{\gamma^*}\,, &  \alpha > \gamma^*
        \end{cases}\,,
    \end{align}
    where $\gamma^*=-\frac{\gamma+1}{2}$ is the critical value.

    Consider
    \[
        g_{j}(\ub):=\int_{\sbb^d}\tilde{f}(\vb)L_{j}(\ub^{\top}\vb)d\tau_{d}(\vb)\,,
    \]
    where
    \[
    L_{j}(t):=\sum_{i=0}^{\infty}\eta\left(\frac{i}{j}\right)N(d,i)P_{i}(t),\quad j\in\nbb_0\,,
    \]
    Here $\eta$ is a $C^{\infty}$ function on $[0,\infty)$ such that $\eta(t)=1$ for $0\leq t\leq 1$ and $\eta(t)=0$ for $t\geq 2$. We claim that $g_j\in\widetilde{\fcal}_{2,\tau_d,s}(\sbb^d)$ and that, for sufficiently large $j$, $g_j$ approximates $\tilde{f}$. Since $\eta$ is supported on $[0,2]$, the sum can be terminated at $i=2j-1$, so $g_{j}$ is a polynomial of degree at most $2j-1$. Moreover, $(g_j)_i=\eta(i/j)\tilde f_i$. The parity assumption gives $(g_j)_i=0$ whenever $i\equiv s\pmod 2$; in particular, all high-frequency inadmissible modes vanish. Hence $g_j$ has only spherical-harmonic components that are admissible for the $\widetilde{\fcal}_{2,\tau_d,s}$ norm. Finally, \cite[Theorem 10.3.2]{2013Approximation} shows that
    \[
    \omega_{\alpha}(\tilde{f},j^{-1})_{2}\asymp\|\tilde{f}-g_{j}\|_{L^{2}(\sbb^d)}+j^{-\alpha}\|(-\Delta)^{\alpha/2}g_{j}\|_{L^{2}(\sbb^d)}\,.
    \]
    Substituting the estimate \eqref{eq:modulibound} into the above, we obtain
    \[
    \|(-\Delta)^{\alpha/2}g_{j}\|_{L^{2}(\sbb^d)}\lesssim\begin{cases}
        1\,, & \alpha<\gamma^*\\
        \sqrt{\log j}\,, & \alpha=\gamma^*\\
        j^{\alpha-\gamma^*}\,, &\alpha>\gamma^*
    \end{cases}\,.
    \]
    Set $\alpha^*=s+\frac{d+1}{2}$. By \eqref{eq:eigenvalue}, $((-\Delta)^{\alpha^*/2}g_{j})_i=(i(i+d-1))^{\alpha^*/2}(g_{j})_i$. We can estimate the norm $\|g_j\|_{\widetilde{\fcal}_{2,\tau_d,s}}$ as follows:
    \begin{align*}
        \|g_j\|_{\widetilde{\fcal}_{2,\tau_d,s}}^2 =& \sum_{\lambda_i\neq0} \frac{1}{\lambda_i^2}\|(g_{j})_i\|_{L^2(\sbb^d)}^2 \\
        \lesssim &\lambda_0^{-2}\|(g_j)_0\|_{L^2(\sbb^d)}^2+\sum_{i=1}^{2j-1}i^{d+2s+1}i^{-2\alpha^*}\|((-\Delta)^{\alpha^*/2}g_j)_i\|_{L^2(\sbb^d)}^2\\
        \lesssim &1+\sum_{i=1}^{2j-1}i^{d+2s+1-2\alpha^*}\|((-\Delta)^{\alpha^*/2}g_j)_i\|_{L^2(\sbb^d)}^2\\
        \lesssim &1 + \|(-\Delta)^{\alpha^*/2}g_j\|_{L^2(\sbb^d)}^2\,,
    \end{align*}
    where the last inequality follows from the choice of $\alpha^*$ and the Parseval identity \eqref{eq:parseval}.
    Since $\gamma^*=s+\frac d2-\zeta(p)$, we have
    \[
        \alpha^*-\gamma^*=\zeta(p)+\frac12>0\,.
    \]
    Taking $\alpha=\alpha^*$ in the modulus estimate gives
    \[
        \|\tilde f-g_j\|_{L^2(\sbb^d)}
        \lesssim \omega_{\alpha^*}(\tilde f,j^{-1})_2
        \lesssim j^{-\gamma^*}\,,
    \]
    and the preceding norm estimate gives
    \[
        \|g_j\|_{\widetilde{\fcal}_{2,\tau_d,s}}
        \lesssim j^{\alpha^*-\gamma^*}
        =j^{\zeta(p)+\frac12}\,.
    \]
    For $R\ge1$, choose $j\asymp R^{1/(\zeta(p)+1/2)}$. Then $\|g_j\|_{\widetilde{\fcal}_{2,\tau_d,s}}\lesssim R$, and
    \[
        \|\tilde f-g_j\|_{L^2(\sbb^d)}
        \lesssim R^{-\frac{\gamma^*}{\zeta(p)+1/2}}
        =R^{-\frac{s+\frac d2-\zeta(p)}{\zeta(p)+\frac12}}\,.
    \]
    Substituting the two cases in the definition of $\zeta(p)$ gives the two exponents stated in Lemma~\ref{lem:approxonsphere}. If $0<R<1$, the same bound follows by taking $g=0$ and using $\|\tilde f\|_{L^2(\sbb^d)}\lesssim1$, which follows from $\|a\|_{L^1(\sbb^d)}\le\|a\|_{L^p(\sbb^d)}\le1$ and the boundedness of $\sigma_s$ on $\sbb^d$. Restoring the factor $\|\tilde f\|_{\fcalwd(\sbb^d)}$ gives the stated estimate.
    This completes the proof of Lemma~\ref{lem:approxonsphere}.
\hfill\BlackBox\\[2mm]

Lemma~\ref{lem:approxonsphere} approximates functions in $\fcalwd(\sbb^d)$ by functions in $\widetilde{\fcal}_{2,\tau_d,s}(\sbb^d)$. The key point is to use the boundedness theorem for spherical-harmonic projection operators to convert the $L^p$ integrability of the weight $a$ into spectral decay, thereby estimating the modulus of smoothness. We next prove the lifting step from the sphere to $\bbb^d$.

\par\noindent{\bf Proof of Corollary~\ref{cor:fromspheretoball}.}
    Denote
    \[
        \sbb_{\mathrm{cap}}^d=\left\{\ub=(\ub',u_{d+1})\in\sbb^d:u_{d+1}\ge\frac{\sqrt2}{2}\right\}\,.
    \]
    For $f\in\fcalwd$, first lift it to this spherical cap by
    \[
        \widetilde f(\ub):=u_{d+1}^s f\left(\frac{\ub'}{u_{d+1}}\right)\,,
        \qquad \ub=(\ub',u_{d+1})\in\sbb_{\mathrm{cap}}^d\,.
    \]
    By the Whitney extension theorem \citep{Fefferman2006Whitney}, this cap-defined function can be extended to the whole sphere. We keep the notation $\widetilde f$ for the extension, choose it to agree with the above lift on $\sbb_{\mathrm{cap}}^d$, and arrange, by multiplying by a smooth cutoff and then taking an odd extension when $s$ is even and an even extension when $s$ is odd, that $\widetilde f(-\ub)=(-1)^{s+1}\widetilde f(\ub)$. The extension may be chosen with
    \[
        \|\widetilde f\|_{\fcalwd(\sbb^d)}\lesssim_{d,s}\|f\|_{\fcalwd}\,.
    \]

    For $\xb\in\bbb^d$, put
    \[
        \ub_{\xb}:=\frac{1}{\sqrt{\|\xb\|_2^2+1}}\begin{pmatrix}\xb\\1\end{pmatrix}\in\sbb_{\mathrm{cap}}^d\,.
    \]
    Conversely, for $g\in\widetilde{\fcal}_{2,\tau_d,s}(\sbb^d)$, define the reverse mapping
    \begin{align*}
        g^*(\xb) := \left(\|\xb\|_2^2+1\right)^{\frac{s}{2}}g(\ub_{\xb}), \quad \xb\in\bbb^d\,.
    \end{align*}
    One can verify that $g^*\in\widetilde{\fcal}_{2,\tau_d,s}(\Omega)$. Since $\ub_{\xb}$ ranges over a spherical cap on which the surface measure is equivalent to the pushforward of Lebesgue measure on $\bbb^d$, and since $(1+\|\xb\|_2^2)^{s/2}$ is bounded on $\bbb^d$, we have
    \begin{align*}
        \|f-g^*\|_{L^2(\Omega)}\lesssim_{d,s}\|\tilde{f}-g\|_{L^2(\sbb^d)}\,,
    \end{align*}
    which, combined with Lemma~\ref{lem:approxonsphere}, yields the conclusion. This completes the proof of Corollary~\ref{cor:fromspheretoball}.
\hfill\BlackBox\\[2mm]
\par\noindent{\bf Proof of Theorem~\ref{thm:var-approx}.}
    By homogeneity, assume $\|f\|_{\fcalwd}=1$. Put
    \[
        A=2s+d+1,\qquad q=\frac{A}{2d}\,.
    \]
    Let $g\in\widetilde{\fcal}_{2,\tau_d,s}$ satisfy $\|g\|_{\widetilde{\fcal}_{2,\tau_d,s}}\le R$. Choosing a representation arbitrarily close to the infimum in the $\widetilde{\fcal}_{2,\tau_d,s}$ norm, we may write it with a coefficient $a\in L^2(\tau_d)$ satisfying $\|a\|_{L^2(\tau_d)}\le R$. The normalization of $\tau_d$ gives $\|a\|_{L^1(\tau_d)}\le\|a\|_{L^2(\tau_d)}$. After the harmless normalization of the parameters by the positive homogeneity of $\sigma_s$, this gives an $L^1$ integral representation in the variation space generated by ReLU$^s$ atoms. Hence, by Lemma~\ref{lem:kpsequivbs}, $\|g\|_{\bscr_s}\lesssim R$. Applying Lemma~\ref{thm:barron-approx} and using $\|h\|_{L^2(\Omega)}\le|\Omega|^{1/2}\|h\|_{L^\infty(\Omega)}$, we can find $f_m\in\Sigma_m^s$ such that
    \begin{align}\label{eq:F2-network-approx}
        \|g-f_m\|_{L^2(\Omega)}\lesssim Rm^{-q}\,.
    \end{align}
    Combining \eqref{eq:F2-network-approx} with Corollary~\ref{cor:fromspheretoball}, it remains to optimize over $R$.

    If $1\le p\le p^*$, then
    \[
        \|f-f_m\|_{L^2(\Omega)}
        \lesssim R^{-a}+Rm^{-q},\qquad
        a=\frac{p(2s+2d+1)-2d}{d(2-p)}\,.
    \]
    Taking $R\asymp m^{q/(a+1)}$, that is,
    \[
        R\asymp m^{\frac{2-p}{2p}}\,,
    \]
    gives
    \[
        \|f-f_m\|_{L^2(\Omega)}
        \lesssim m^{-\frac{qa}{a+1}}
        =m^{-\frac{p(2s+2d+1)-2d}{2dp}}\,.
    \]

    If $p^*<p<2$, then
    \[
        \|f-f_m\|_{L^2(\Omega)}
        \lesssim R^{-a}+Rm^{-q},\qquad
        a=\frac{p(4s+3d-1)-2d+2}{2d-2-pd+3p}\,.
    \]
    Taking $R\asymp m^{q/(a+1)}$, that is,
    \[
        R\asymp m^{\frac{2d-2-pd+3p}{4dp}}\,,
    \]
    gives
    \[
        \|f-f_m\|_{L^2(\Omega)}
        \lesssim m^{-\frac{qa}{a+1}}
        =m^{-\frac{p(4s+3d-1)-2d+2}{4dp}}\,.
    \]
    Restoring the factor $\|f\|_{\fcalwd}$ completes the proof of Theorem~\ref{thm:var-approx}.
\hfill\BlackBox\\[2mm]

\section{Proof of Theorem \ref{thm:sobolev-approx}}\label{appendix:proof_sobolev}
We prove Theorem \ref{thm:sobolev-approx} by combining an embedding into spectral Barron spaces with the corresponding shallow-network approximation result. Let $\tscr(\rbb^d)$\label{sym:first:tscr} and $\tscr'(\rbb^d)$\label{sym:first:tscr_prime} denote the Schwartz space and the space of tempered distributions on $\rbb^d$, respectively, with the Fourier transform understood in the distributional sense. For $r\ge0$, define the \emph{spectral Barron space of order $r$} as follows:

\begin{align*}
    \fscr_r=\left\{f(\xb)=\int_{\rbb^d}e^{i\xb\cdot\xi}\hat{F}(d\xi):\hat{F}\text{ is a finite complex measure},\|f\|_{\fscr_r}<\infty\right\}\,.
\end{align*}
\label{sym:first:fscr_s}
Here,
\begin{align*}
    \|f\|_{\fscr_r}=\inf_{f_e|_{\Omega}=f}\int_{\rbb^d}(1+\|\xi\|_{1})^r|\hat{F}_e(d\xi)|\,,
\end{align*}
\label{sym:first:fscr_norm}
where the infimum is taken over all $\tscr'(\rbb^d)$ extensions of $f$. We use the following special case of \cite[Theorem~3]{SiegelXu2022HighOrder}: for $s\in\nbb$ and $g\in\fscr_s$,
\begin{align}\label{eq:approximation2}
    \inf_{g_m\in\Sigma_m^s}\|g-g_m\|_{L^2(\Omega)}
    \lesssim_{s,d,\Omega}\|g\|_{\fscr_s}m^{-\frac{d+2s}{2(d+1)}}\,.
\end{align}
Indeed, in that theorem one takes the spectral Barron order and the ReLU order both equal to $s$, and the Sobolev error order equal to zero. The theorem is stated with the Euclidean Fourier weight on a cube; this is equivalent to the $\ell_1$ Fourier weight used above, up to constants depending only on $d$. Applying the theorem on a cube containing $\Omega$ and then restricting to $\Omega$ gives \eqref{eq:approximation2}. If the approximant has complex coefficients and $g$ is real-valued, taking its real part gives a real-valued ReLU$^s$ network with no larger $L^2$ error.
Furthermore, \cite{Liao_2025} and \cite{choulli2025functionalanalysispartialdifferential} established the embedding relationship between Sobolev spaces and spectral Barron spaces: for $r\in\nbb_0$, $1\le p\le2$, if $\alpha>r+\frac{d}{p}$, then
\begin{align*}
    W^{\alpha,p}\hookrightarrow\fscr_r\,.
\end{align*}
\label{sym:first:hookrightarrow}
For the critical case $p=1$ and $\alpha=r+d$, the above embedding also holds, i.e.,
\begin{align*}
    W^{r+d,1}\hookrightarrow\fscr_r\,.
\end{align*}
We use this embedding to prove the approximation theorem for Sobolev spaces, Theorem~\ref{thm:sobolev-approx}.
\par\noindent{\bf Proof of Theorem~\ref{thm:sobolev-approx}.}
    First consider the case $\alpha\ge s+d$. By the classical Sobolev extension theorem, $f$ can be extended to $\rbb^d$ with $\|f\|_{W^{\alpha,p}(\rbb^d)}\lesssim_\Omega \|f\|_{W^{\alpha,p}}$. Since $1\le p<2$, the embedding stated above, applied with $r=s$, gives $W^{\alpha,p}\hookrightarrow\fscr_s$; for $p=1$ and $\alpha=s+d$, this is exactly the critical embedding $W^{s+d,1}\hookrightarrow\fscr_s$. Hence $\|f\|_{\fscr_s}\lesssim_\Omega\|f\|_{W^{\alpha,p}}$. Applying \eqref{eq:approximation2} and using $\|h\|_{L^p(\Omega)}\le|\Omega|^{1/p-1/2}\|h\|_{L^2(\Omega)}$, we obtain
    \begin{align*}
        \inf_{f_m\in\Sigma_m^s}\|f-f_m\|_{L^p(\Omega)}
        \lesssim \|f\|_{W^{\alpha,p}}m^{-\frac{d+2s}{2(d+1)}}\,.
    \end{align*}

    It remains to prove the case $\alpha\in\nbb\cap(0,s+d)$. Assume $\Omega$ is contained in a ball of radius $R>0$, i.e., $\Omega\subset\bbb^d_R=\{\xb\in\rbb^d:\|\xb\|_2\le R\}$. Again by the Sobolev extension theorem and a smooth cutoff equal to one on $\Omega$, $f$ can be extended to a compactly supported function on $\rbb^d$ with $\|f\|_{W^{\alpha,p}(\rbb^d)}\lesssim_\Omega \|f\|_{W^{\alpha,p}}$.

    Let $\phi:\rbb^d\to[0,+\infty)$ be a $C^\infty$ radially symmetric function supported in the unit ball, satisfying $\int_{\rbb^d}\phi(\xb)d\xb=1$. For any $\varepsilon>0$, define $\phi_\varepsilon(\xb)=\varepsilon^{-d}\phi(\xb/\varepsilon)$. Then $\phi_\varepsilon$ is an approximate identity, and by the triangle inequality and the normalization of $\phi$,
    \begin{align*}
       \|\phi_\varepsilon*f\|_{W^{\alpha,p}(\rbb^d)}\le \|f\|_{W^{\alpha,p}(\rbb^d)}\,.
    \end{align*}
    By a standard density argument, it is enough to prove the estimate for $f\in C_c^\infty(\rbb^d)$; the constants below depend only on the extension domain and not on the particular smooth approximant.

    Now let $\varepsilon>0$ be a parameter to be determined. Define the approximant
    \begin{align}\label{eq:fvarepsilon}
        f_\varepsilon(\xb)=\sum_{t=1}^\alpha\binom{\alpha}{t}(-1)^{t-1}\int_{\rbb^d} \phi_\varepsilon(\yb)f(\xb-t\yb)d\yb\,.
    \end{align}
    Using the normalization of $\phi_\varepsilon$, we estimate $\|f-f_\varepsilon\|_{L^p(\rbb^d)}$:
    \begin{align*}
        \|f-f_\varepsilon\|_{L^p(\rbb^d)}&\le\left\|\int_{\rbb^d}\phi_\varepsilon(\yb)\sum_{t=0}^\alpha\binom{\alpha}{t}(-1)^tf(\xb-t\yb)d\yb\right\|_{L^p(\rbb^d)}\,.
    \end{align*}
    For a fixed $\yb\in\rbb^d$, we estimate
    \begin{align*}
        \left\|\left(\sum_{t=0}^\alpha\binom{\alpha}{t}(-1)^tf(\xb-t\yb)\right)\right\|_{L^p(\rbb^d)}&=\|\Delta_\yb^\alpha f\|_{L^p(\rbb^d)}\,,
    \end{align*}
    where $\Delta_\yb^\alpha$ is the $\alpha$-th order difference operator of $f$. By the fundamental theorem of calculus,
    \begin{align*}
        \Delta_\yb f(\xb)=-\int_0^1\nabla f(\xb-t\yb)\cdot \yb dt\,.
    \end{align*}
    Since the difference operator $\Delta_\yb$ commutes with integration over $t$, we have
    \begin{align*}
        \Delta_\yb^\alpha f(\xb)=(-1)^\alpha\int_{[0,1]^\alpha} D^\alpha f(\xb-(t_1+\cdots t_\alpha)\yb)\cdot \yb^{\otimes \alpha}dt_1\cdots dt_\alpha\,.
    \end{align*}
    Expanding the tensor contraction into mixed partial derivatives, there exists a constant $C_\alpha$ such that
    \begin{align*}
        |\Delta_\yb^\alpha f(\xb)|\le C_\alpha|\yb|^\alpha\int_{[0,1]^\alpha}|D^\alpha f(\xb-(t_1+\cdots t_\alpha)\yb)|dt_1\cdots dt_\alpha\,,
    \end{align*}
    where $C_\alpha = \alpha^{\alpha} \binom{\alpha+d-1}{d-1}$ is an upper bound for the polynomial coefficients. Since $1\le p<2$, Jensen's inequality on $[0,1]^\alpha$ gives
    \begin{align*}
        |\Delta_\yb^\alpha f(\xb)|^p
        &\le C_\alpha^p|\yb|^{\alpha p}\int_{[0,1]^\alpha}|D^\alpha f(\xb-(t_1+\cdots t_\alpha)\yb)|^pdt_1\cdots dt_\alpha\,,
    \end{align*}
    where we used the fact that $[0,1]^\alpha$ has measure one.
    Integrating over $\xb$ and swapping the order of integration yields
    \begin{align*}
        \|\Delta_\yb^\alpha f\|_{L^p(\rbb^d)}^p &\le C_\alpha^p|\yb|^{\alpha p}\int_{[0,1]^\alpha}\int_{\rbb^d}|D^\alpha f(\xb-(t_1+\cdots t_\alpha)\yb)|^pd\xb dt_1\cdots dt_\alpha\\
        &= C_\alpha^p|\yb|^{\alpha p} \|f\|_{W^{\alpha,p}(\rbb^d)}^p\,.
    \end{align*}
    Therefore,
    \begin{align*}
        \|\Delta_\yb^\alpha f\|_{L^p(\rbb^d)}\le C_\alpha|\yb|^\alpha\|f\|_{W^{\alpha,p}(\rbb^d)}\,.
    \end{align*}
    Since $\phi_\varepsilon$ is supported in $B_{\varepsilon }$, we have
    \begin{align*}
        \|f-f_\varepsilon\|_{L^p(\rbb^d)}&\le \int_{\rbb^d}\phi_\varepsilon(\yb)\|\Delta_\yb^\alpha f\|_{L^p(\rbb^d)}d\yb\\
        &\le C_\alpha\|f\|_{W^{\alpha,p}(\rbb^d)}\int_{B_{\varepsilon}}\phi_\varepsilon(\yb)|\yb|^\alpha d\yb\\
        &\le C_\alpha\|f\|_{W^{\alpha,p}(\rbb^d)}\varepsilon^\alpha\,,
    \end{align*}
    where the last step uses the fact that $\phi_\varepsilon$ is supported in $B_{\varepsilon}$ and $\int\phi_\varepsilon=1$.

    Next, we estimate $\|f_\varepsilon\|_{W^{\alpha_0,1}(\rbb^d)}$, where $\alpha_0=s+d$. Since $\alpha$ is fixed, it suffices to estimate
    \begin{align*}
        \left\|\int_{\rbb^d}\phi_\varepsilon(\yb)f(\xb-t\yb)d\yb\right\|_{W^{\alpha_0,1}(\rbb^d)},\ \ t=1,\cdots,\alpha\,.
    \end{align*}
    To this end, perform a change of variables:
    \begin{align*}
        f_{\varepsilon,t}(\xb)=\frac{1}{t^d}\int_{\rbb^d}\phi_\varepsilon\left(\frac{\yb}{t}\right)f(\xb-\yb)d\yb=\int_{\rbb^d}\phi_{t\varepsilon}(\yb)f(\xb-\yb)d\yb\,.
    \end{align*}
    Consider
    \begin{align*}
        \|f_{\varepsilon,t}\|_{W^{\alpha_0,1}(\rbb^d)}=\sum_{|\beta|\le \alpha_0}\|D^\beta f_{\varepsilon,t}\|_{L^1(\rbb^d)}=\sum_{|\beta|\le\alpha_0}\|D^\beta (\phi_{t\varepsilon}*f)\|_{L^1(\rbb^d)}\,.
    \end{align*}
    For a fixed multi-index $\beta$ with $|\beta|\le\alpha_0$, choose a multi-index $\gamma\le\beta$ as follows: if $|\beta|\le\alpha$, set $\gamma=\beta$; otherwise choose $\gamma\le\beta$ with $|\gamma|=\alpha$. Then $|\beta-\gamma|=(|\beta|-\alpha)_+$. Since $\phi$ is a Schwartz function and $f$ has compact support, we can interchange convolution and weak derivatives to obtain
    \begin{align*}
        D^\beta (\phi_{t\varepsilon}*f)=(D^{\beta-\gamma}\phi_{t\varepsilon})*(D^\gamma f)\,,
    \end{align*}
    By Young's convolution inequality and the compact support of $f$, we have
    \begin{align*}
        \|D^\beta (\phi_{t\varepsilon}*f)\|_{L^1(\rbb^d)}&\le \|D^{\beta-\gamma}\phi_{t\varepsilon}\|_{L^1(\rbb^d)}\|D^\gamma f\|_{L^1(\rbb^d)}\\
        &\le C_R \|D^{\beta-\gamma}\phi_{t\varepsilon}\|_{L^1(\rbb^d)}\|f\|_{W^{\alpha,p}(\rbb^d)}\,,
    \end{align*}
    where $C_R$ comes from the compact support of $f$ and Hölder's inequality.
    Direct computation gives
    \begin{align*}
        D^{\beta-\gamma}\phi_{t\varepsilon}(\xb)=(t\varepsilon)^{-d-|\beta-\gamma|}(D^{\beta-\gamma}\phi)(\xb/(t\varepsilon))\,,
    \end{align*}
    so after the change of variables $\xb\mapsto (t\varepsilon)\xb$, we have
    \begin{align*}
        \|D^{\beta-\gamma}\phi_{t\varepsilon}\|_{L^1(\rbb^d)}&=(t\varepsilon)^{-d-(|\beta-\gamma|)+d}\|D^{\beta-\gamma}\phi\|_{L^1(\rbb^d)}\\
        &=(t\varepsilon)^{-|\beta-\gamma|}\|D^{\beta-\gamma}\phi\|_{L^1(\rbb^d)}\,.
    \end{align*}
    Let $M_\phi:=\max_{|\delta|\le\alpha_0}\|D^\delta\phi\|_{L^1(\rbb^d)}<\infty$. Then
    \begin{align*}
        \|D^\beta (\phi_{t\varepsilon}*f)\|_{L^1(\rbb^d)}&\le C_R M_\phi (t\varepsilon)^{-|\beta-\gamma|}\|f\|_{W^{\alpha,p}(\rbb^d)}\\
        &= C_R M_\phi (t\varepsilon)^{-(|\beta|-\alpha)_+}\|f\|_{W^{\alpha,p}(\rbb^d)}\,.
    \end{align*}
    Therefore, summing over $\beta$ and noting that $0\le(|\beta|-\alpha)_+\le\alpha_0-\alpha$, we have, whenever $\varepsilon\le1$,
    \begin{align*}
        \|f_{\varepsilon,t}\|_{W^{\alpha_0,1}(\rbb^d)}&=\sum_{|\beta|\le\alpha_0}\|D^\beta (\phi_{t\varepsilon}*f)\|_{L^1(\rbb^d)}\\
        &\le C_R M_\phi \|f\|_{W^{\alpha,p}(\rbb^d)}\sum_{|\beta|\le\alpha_0}\varepsilon^{-(\alpha_0-\alpha)}\\
        &= C_R M_\phi \binom{\alpha_0+d}{d}\varepsilon^{-(\alpha_0-\alpha)}\|f\|_{W^{\alpha,p}(\rbb^d)}\,.
    \end{align*}
    Hence,
    \begin{align*}
        \|f_{\varepsilon,t}\|_{W^{\alpha_0,1}(\rbb^d)}
        &\le C_{\alpha,d}\varepsilon^{-(\alpha_0-\alpha)}\|f\|_{W^{\alpha,p}(\rbb^d)}\,,
    \end{align*}
    where $C_{\alpha,d}=C_R M_\phi \binom{\alpha_0+d}{d}=C_R M_\phi \binom{s+2d}{d}$.
    From the arbitrariness of $t$ ($t=1,\ldots,\alpha$) and the definition of $f_\varepsilon$ in \eqref{eq:fvarepsilon}, we obtain
    \begin{align*}
        \|f_\varepsilon\|_{W^{\alpha_0,1}(\rbb^d)}&\le \sum_{t=1}^\alpha\binom{\alpha}{t}\|f_{\varepsilon,t}\|_{W^{\alpha_0,1}(\rbb^d)}\\
        &\le C_R M_\phi \binom{s+2d}{d}\varepsilon^{-(\alpha_0-\alpha)}\|f\|_{W^{\alpha,p}(\rbb^d)}\sum_{t=1}^\alpha\binom{\alpha}{t}\\
        &= C_R M_\phi \binom{s+2d}{d}(2^\alpha-1)\varepsilon^{-(\alpha_0-\alpha)}\|f\|_{W^{\alpha,p}(\rbb^d)}\\
        &\le C_{\alpha,d}'(s,d)\|f\|_{W^{\alpha,p}(\rbb^d)}\varepsilon^{-(\alpha_0-\alpha)}\,,
    \end{align*}
    where $C_{\alpha,d}'(s,d)=C_R M_\phi \binom{s+2d}{d}(2^\alpha-1)$; we do not track its dimension dependence.
    By the embedding $W^{\alpha_0,1}\hookrightarrow\fscr_s$ and the approximation theorem \eqref{eq:approximation2}, there exists $f_m\in\Sigma_{m}^s$ such that
    \begin{align*}
        \|f_\varepsilon-f_m\|_{L^p(\Omega)}&\le C_0(s,d,\Omega) \|f_\varepsilon\|_{W^{\alpha_0,1}(\rbb^d)}m^{-\frac{d+2s}{2(d+1)}}\\
        &\le C_0(s,d,\Omega) \cdot C_{\alpha,d}'(s,d)\|f\|_{W^{\alpha,p}(\rbb^d)}\varepsilon^{-(\alpha_0-\alpha)}m^{-\frac{d+2s}{2(d+1)}}\,,
    \end{align*}
    where $C_0(s,d,\Omega)$ includes the factor converting the $L^2(\Omega)$ bound in \eqref{eq:approximation2} to an $L^p(\Omega)$ bound.
    Consequently, by the triangle inequality,
    \begin{align*}
        \|f-f_m\|_{L^p(\Omega)}&\le \|f-f_\varepsilon\|_{L^p(\rbb^d)}+\|f_\varepsilon-f_m\|_{L^p(\Omega)}\\
        &\le C_\alpha \|f\|_{W^{\alpha,p}(\rbb^d)}\varepsilon^\alpha+C_0(s,d,\Omega) C_{\alpha,d}'(s,d)\|f\|_{W^{\alpha,p}(\rbb^d)}\varepsilon^{-(\alpha_0-\alpha)}m^{-\frac{d+2s}{2(d+1)}}\\
        &\le C(s,d,R,\alpha) \|f\|_{W^{\alpha,p}(\rbb^d)}\left(\varepsilon^\alpha+\varepsilon^{-(\alpha_0-\alpha)}m^{-\frac{d+2s}{2(d+1)}}\right)\,,
    \end{align*}
    Now choose $\varepsilon=m^{-\frac{d+2s}{2(s+d)(d+1)}}$; for finitely many small $m$ with $\varepsilon>1$, the estimate is absorbed into the constant. Substituting this choice yields
    \begin{align*}
        \|f-f_m\|_{L^p(\Omega)}\lesssim \|f\|_{W^{\alpha,p}}m^{-\frac{\alpha(d+2s)}{2(s+d)(d+1)}}\,.
    \end{align*}
    This completes the proof of Theorem~\ref{thm:sobolev-approx}.
\hfill\BlackBox\\[2mm]

\section{Proofs in Section~\ref{sec:generalization}}\label{appendix:proof-upper}
We provide the proofs of the generalization results from Section~\ref{sec:generalization}.
\subsection*{Proof of Theorem~\ref{thm:upperbound}}
Before proving this theorem, we need some preparatory results. Define the empirical distribution associated with the sample $\{(\xb_i,y_i)\}_{i=1}^n$ by
\begin{align*}
\mu_n=\frac{1}{n}\sum_{i=1}^n\delta_{\xb_i}\,,
\end{align*}
\label{sym:first:delta_x}
which induces the corresponding empirical $L^2(\mathrm{d} \mu_n)$ norm
\begin{align*}
    \|f\|_{L^2(\mathrm{d} \mu_n)}^2=\int_\Omega |f(\xb)|^2\mathrm{d} \mu_n(\xb)=\frac{1}{n}\sum_{i=1}^n|f(\xb_i)|^2\,.
\end{align*} 
The upper bound for the generalization error depends on the size of the chosen neural-network model class, so we need a notion of model complexity. The local complexity defined below is the main tool in the error analysis.

\begin{definition}
    Given a model class $\fcal$, define its local complexity by
    \begin{align*}
        \gcal_n(\fcal,\delta,{\xib})=\ebb_{\xib}\left[\sup_{f\in\fcal,\|f\|_{L^2(\mathrm{d} \mu_n)}\le \delta}\left|\frac{1}{n}\sum_{i=1}^n\xi_if(\xb_i)\right|\right],\ \ \delta>0\,,
    \end{align*}
    \label{sym:first:gcal_n}
    where $\{\xi_i\}_{i=1}^n$ is a sequence of independent identically distributed zero-mean random variables. 
\end{definition}

When $\xi_i=\eta_i$, we abbreviate it as $\gcal_n(\fcal,\delta)$; this quantity measures the correlation between the model class $B_{\partial\Sigma_m^s}(r)$ and the noise ${\etab}$ on the sample $\{\xb_i\}_{i=1}^n$. When $\xi_i$ are Rademacher or Gaussian random variables, $\gcal_n(\fcal,\delta,{\xib})$ reduces to the classical local Rademacher or local Gaussian complexity, respectively \citep{Wainwright_2019}. The value of $\gcal_n(\delta,{\xib})$ depends on the sample $\{\xb_i\}_{i=1}^n$, but we suppress this dependence because the complexity bounds below hold uniformly over all samples.

A set is called \emph{star-shaped} if for any $f$ in the set and any $t\in[0,1]$, we also have $tf$ in the set. For a non-star-shaped set $\fcal$, we may define its star-shaped hull by
\begin{align*}
   \starhull(\fcal)=\{tf:f\in\fcal,t\in[0,1]\}\,.
\end{align*}
\label{sym:first:starhull}
\begin{lemma}
    Let $\fcal$ be star-shaped. Viewing $\frac{\gcal_n(\fcal,\delta,{\xib})}{\delta}$ as a function of $\delta$, it is monotonically decreasing on $(0,+\infty)$. Therefore, for any constant $c>0$,
    \begin{align*}
        \delta_n^*=\min\{\delta>0:\gcal_n(\fcal,\delta,{\xib})\le c\delta^2\}
    \end{align*}
    always exists and is finite. 
\end{lemma}
\begin{proof}
    Write the local complexity simply as $\gcal_n(\delta)$. For any $0<\delta\le t$, we show that $\gcal_n(t)\le\frac{t}{\delta}\gcal_n(\delta)$. For any $f\in\fcal$ with $\|f\|_{L^2(\mathrm{d} \mu_n)}\le t$, by star-shapedness we have $\frac{\delta}{t}f\in\fcal$ and $\left\|\frac{\delta}{t}f\right\|_{L^2(\mathrm{d} \mu_n)}\le \delta$. Hence
    \begin{align*}
        \left|\frac{1}{n}\sum_{i=1}^n\xi_if(\xb_i)\right|=&\frac{t}{\delta}\left|\frac{1}{n}\sum_{i=1}^n\xi_i\cdot\frac{\delta}{t}f(\xb_i)\right|\le\frac{t}{\delta}\sup_{f\in\fcal,\|f\|_{L^2(\mathrm{d} \mu_n)}\le \delta}\left|\frac{1}{n}\sum_{i=1}^n\xi_if(\xb_i)\right|\,.
    \end{align*}
    By the arbitrariness of $f$, it follows that
    \begin{align*}
        \sup_{f\in\fcal,\|f\|_{L^2(\mathrm{d} \mu_n)}\le t}\left|\frac{1}{n}\sum_{i=1}^n\xi_if(\xb_i)\right|\le\frac{t}{\delta}\sup_{f\in\fcal,\|f\|_{L^2(\mathrm{d} \mu_n)}\le \delta}\left|\frac{1}{n}\sum_{i=1}^n\xi_if(\xb_i)\right|\,.
    \end{align*}
    Taking expectation with respect to ${\xib}$ yields the desired conclusion. This completes the proof of the lemma.
\end{proof}

As we shall see in the subsequent proof, estimating the size of $\delta_n^*$ is the key to analyzing the upper bound of the generalization error, and for this purpose we need an upper bound on the local complexity. Dudley's entropy integral inequality \citep{Wainwright_2019} is a classical tool for estimating local complexity. The following theorem is in fact an application of Dudley's entropy integral inequality, and its proof relies on the classical chaining method.
\begin{theorem}\label{thm:localcomplexity}
    Let $M>0$, and denote by $\bscr_s(M)$ the closed ball of radius $M$ in the $s$th-order Barron space. Then for any $0<\delta\le M$ and any independent identically distributed sub-Gaussian random variables ${\xib}$, we have
    \begin{align*}
        \gcal_n(\bscr_s(M),\delta,{\xib})\lesssim \delta^{\frac{2s+1}{d+2s+1}}M^{\frac{d}{d+2s+1}}n^{-\frac{1}{2}}\sqrt{\log\frac{nM}{\delta}}\,.
    \end{align*}
    The implied constant is independent of the sample. The same upper bound also holds for the function class $\starhull(\pi_B(\bscr_s(M)))$\label{sym:first:pi_B_fcal}.
\end{theorem}
\begin{proof}
    For a model class $\fcal$, let $N=\mathcal{N}(\varepsilon,\fcal,\|\cdot\|_{L^2(\mathrm{d} \mu_n)})$ denote its covering number\label{sym:first:covering_number}, that is, there exist $\{f_1,\cdots,f_N\}\subset\fcal$ such that for any $f\in\fcal$, there is some $f_j$ satisfying $\|f-f_j\|_{L^2(\mathrm{d} \mu_n)}\le\varepsilon$. Therefore, by the Cauchy--Schwarz inequality,
    \begin{align*}
        \left|\frac{1}{n}\sum_{i=1}^n\eta_if(\xb_i)\right|\le&\left|\frac{1}{n}\sum_{i=1}^n\eta_if_j(\xb_i)\right|+\left|\frac{1}{n}\sum_{i=1}^n\eta_i(f(\xb_i)-f_j(\xb_i))\right|\\
        \le&\max_{j=1,\cdots,N}\left|\frac{1}{n}\sum_{i=1}^n\eta_if_j(\xb_i)\right|+\sqrt{\frac{\sum_{i=1}^n\eta_i^2}{n}}\varepsilon\,,
    \end{align*}
    and hence, by the arbitrariness of $f$ and taking expectation with respect to the noise,
    \begin{align*}
        \ebb_{\etab}\left[\sup_{f\in\fcal}\left|\frac{1}{n}\sum_{i=1}^n\eta_if(\xb_i)\right|\right]\le&\ebb_{\etab}\left[\max_{j=1,\cdots,N}\left|\frac{1}{n}\sum_{i=1}^n\eta_if_j(\xb_i)\right|\right]+\ebb\left[\sqrt{\frac{\sum_{i=1}^n\eta_i^2}{n}}\varepsilon\right]\\
    \end{align*}
    By Jensen's inequality and the moment assumption on the noise,
    \begin{align*}
        \ebb\left[\sqrt{\frac{\sum_{i=1}^n\eta_i^2}{n}}\varepsilon\right]\le&\sqrt{\ebb\left[\frac{\sum_{i=1}^n\eta_i^2}{n}\right]}\varepsilon\le\sqrt{2CR^2}\varepsilon\,,
    \end{align*}
    and therefore
    \begin{align*}
        \ebb_{\etab}\left[\sup_{f\in\fcal}\left|\frac{1}{n}\sum_{i=1}^n\eta_if(\xb_i)\right|\right]\le&\ebb_{\etab}\left[\max_{j=1,\cdots,N}\left|\frac{1}{n}\sum_{i=1}^n\eta_if_j(\xb_i)\right|\right]+\sqrt{2CR^2}\varepsilon\,.
    \end{align*}
    Define a family of zero-mean random variables by
    \begin{align*}
        Z(f)=\frac{1}{\sqrt{n}}\sum_{i=1}^n\eta_if(\xb_i),\ \ f\in\fcal\,,
    \end{align*}
    Then it is not hard to show that $Z(f)$ forms a sub-Gaussian process with respect to the metric $\rho_Z(f,f')=\|f-f'\|_{L^2(\mathrm{d} \mu_n)}$, namely, for any $f,f'\in\fcal$,
    \begin{align*}
        \|Z(f)-Z(f')\|_{\psi_2}\le K\|f-f'\|_{L^2(\mathrm{d} \mu_n)}\,,
    \end{align*}
    Hence, by the chaining method \cite[Theorem 5.22]{Wainwright_2019},
    \begin{align*}
        \ebb_{\etab}\left[\max_{j=1,\cdots,N}\left|\frac{1}{n}\sum_{i=1}^n\eta_if_j(\xb_i)\right|\right]\le \frac{16}{\sqrt{n}}\int_{\frac{\varepsilon}{4}}^{\frac{D}{2}}\sqrt{\log\mathcal{N}(u,\fcal,\|\cdot\|_{L^2(\mathrm{d} \mu_n)})}du\,.
    \end{align*}
    Here $D=\sup_{f,f'\in\fcal}\rho_Z(f,f')$ denotes the diameter of $\fcal$. Combining the above and using the arbitrariness of $\varepsilon$, we obtain
    \begin{align*}
        \ebb_{\etab}\left[\sup_{f\in\fcal}\left|\frac{1}{n}\sum_{i=1}^n\eta_if(\xb_i)\right|\right]\le&\inf_{\varepsilon\ge0}\left\{4\sqrt{2CR^2}\varepsilon+\frac{16}{\sqrt{n}}\int_{\varepsilon}^{\frac{D}{2}}\sqrt{\log\mathcal{N}(u,\fcal,\|\cdot\|_{L^2(\mathrm{d} \mu_n)})}du\right\}\,.
    \end{align*}

    Now set $\fcal=\{f\in\bscr_s(M):\|f\|_{L^2(\mathrm{d} \mu_n)}\le\delta\}$, so that $D=2\delta$. It remains only to estimate its covering number. By following the proof of \cite[Theorem 5]{yang2024nonparametricregressionusingoverparameterized}, one obtains
    \begin{align*}
        \log\mathcal{N}(\varepsilon,\bscr_s(M),\|\cdot\|_{L^2(\mathrm{d} \mu_n)})\lesssim\left(\frac{\varepsilon}{M}\right)^{-\frac{2d}{d+2s+1}}\log\left(1+\frac{M}{\varepsilon}\right)\,.
    \end{align*}
    Therefore,
    \begin{align*}
        \gcal_n(\bscr_s(M),\delta)\lesssim&\inf_{\varepsilon\ge0}\left\{4\sqrt{2CR^2}\varepsilon+\frac{16}{\sqrt{n}}\int_{\varepsilon}^{\delta}\left(\frac{u}{M}\right)^{-\frac{d}{d+2s+1}}\sqrt{\log\left(1+\frac{M}{u}\right)}du\right\}\\
    \end{align*}
    Taking $\varepsilon\asymp\delta^{\frac{2s+1}{d+2s+1}}M^{\frac{d}{d+2s+1}}n^{-\frac{1}{2}}$, we get
    \begin{align*}
        \gcal_n(\bscr_s(M),\delta)\lesssim&\delta^{\frac{2s+1}{d+2s+1}}M^{\frac{d}{d+2s+1}}n^{-\frac{1}{2}}+M^{\frac{d}{d+2s+1}}n^{-\frac{1}{2}}\sqrt{\log\left(1+\frac{M}{\varepsilon}\right)}\int_0^\delta u^{-\frac{d}{d+2s+1}}du\\
        \lesssim& \delta^{\frac{2s+1}{d+2s+1}}M^{\frac{d}{d+2s+1}}n^{-\frac{1}{2}}\sqrt{\log\frac{nM}{\delta}}\,.
    \end{align*}

    It remains to treat $\fcal=\starhull(\pi_B(\bscr_s(M)))$. By the Lipschitz continuity of $\pi_B$, $\mathcal{N}(\delta,\pi_B(\bscr_s))\lesssim\mathcal{N}(\delta,\bscr_s(M))$. Let $\{f_1,\cdots,f_m\}$ be a $\frac{\delta}{2}$-cover of $\pi_B(\bscr_s(M))$, and let $\{t_1,\cdots,t_L\}$ be a $\frac{\delta}{2B}$-cover of $[0,1]$. For any $f\in\starhull(\pi_B(\bscr_s(M)))$, there exist $t\in[0,1]$ and $g\in\pi_B(\bscr_s(M))$ such that $f=tg$. By the definition of the covering sets, there exist $t_j,g_i$ such that $|t-t_j|\le \frac{\delta}{2B}$ and $\|g-g_i\|_{L^2(\mathrm{d} \mu_n)}\le \frac{\delta}{2}$. Hence
    \begin{align*}
        \|f-t_jg_i\|_{L^2(\mathrm{d} \mu_n)}=&\|tg-t_jg_i\|_{L^2(\mathrm{d} \mu_n)}\\
        \le&\|t(g-g_i)\|_{L^2(\mathrm{d} \mu_n)}+\|(t-t_j)g_i\|_{L^2(\mathrm{d} \mu_n)}\\
        \le&\|g-g_i\|_{L^2(\mathrm{d} \mu_n)}+|t-t_j|\cdot\|g_i\|_{L^2(\mathrm{d} \mu_n)}\\
        \le&\frac{\delta}{2}+B\cdot\frac{\delta}{2B}=\delta\,.
    \end{align*}
    This shows that
    \begin{align*}
        \log\mathcal{N}(\delta,\starhull(\pi_B(\bscr_s(M))),\|\cdot\|_{L^2(\mathrm{d} \mu_n)})\lesssim&\left(\frac{\delta}{M}\right)^{-\frac{d}{d+2s+1}}\log\left(1+\frac{M}{\delta}\right)+\log\left(1+\frac{B}{\delta}\right)\\
        \lesssim&\left(\frac{\delta}{M}\right)^{-\frac{d}{d+2s+1}}\log\left(1+\frac{M}{\delta}\right)\,.
    \end{align*}
    The conclusion follows by the same argument as above. This completes the proof of Theorem~\ref{thm:localcomplexity}.
\end{proof}

Finally, we introduce a quantity finer than complexity, which arises from the generic chaining method \citep{Talagrand21}. Consider a set $T$ equipped with a seminorm $d$. Let $\{A_k\}_{k=0}^\infty$ be an increasing sequence of partitions of $T$, meaning that each $A_i$ is a partition of $T$, and every element in $A_{k+1}$ is contained in some element of $A_k$. If $|A_0|=1$ and for every $k\ge 1$ we have $|A_k|\le 2^{2^k}$, then $\{A_k\}_{k=0}^\infty$ is called an \emph{admissible sequence} of $T$. For $t\in T$, we write $A_n(t)$ for the unique element of $A_n$ containing $t$. Define
\begin{align*}
    \gamma_2(T,d)=\inf\sup_{t\in T}\sum_{k=0}^\infty 2^{\frac{k}{2}}\diam(A_k(t))\,,
\end{align*}
\label{sym:first:gamma_2}
where the infimum is taken over all admissible sequences of $T$. This quantity is closely related to Dudley's entropy integral; in fact,
\begin{align*}
    \gamma_2(T,d)\lesssim \int_0^{\diam(T)}\sqrt{\log \mathcal{N}(u,T,d)}du\,.
\end{align*}
Moreover, if $T=\{f\in\fcal:\|f\|_{L^2(\mathrm{d} \mu_n)}\le\delta\}$ and $d$ is induced by $\|\cdot\|_{L^2(\mathrm{d} \mu_n)}$, then when ${\etab}$ is Gaussian, Talagrand's theorem states that $\gcal_n(\fcal,\delta,{\etab})\asymp\gamma_2(T,d)$. 

Consider $T(\delta)=\{f\in\starhull(\partial(\Sigma_m^s)):\|f\|_{L^2(\mathrm{d} \mu_n)}\le\delta\}$ with $d$ induced by $\|\cdot\|_{L^2(\mathrm{d} \mu_n)}$. Denote
\begin{align*}
    \gamma_2(\delta)=\gamma_2(T(\delta),d)\,.
\end{align*}
Since $\starhull(\partial(\Sigma_m^s))$ is star-shaped, the standard scaling property of $\gamma_2$ implies that $\frac{\gamma_2(\delta)}{\delta}$ is monotonically decreasing on $(0,+\infty)$. 

Before formally deriving the upper bound on the error, we first briefly describe the proof strategy. We begin with a decomposition of the excess error. Let $B=B(m)>\sup_{x\in\Omega}|f_\rho(x)|>0$ be a truncation constant to be determined. For any $f\in\Sigma_m^s$, consider
\begin{align*}
    \|\pi_B f_m-f_\rho\|_{L^2(\mathrm{d} \mu)}^2\le&2\|\pi_B f_m-\pi_B f\|_{L^2(\mathrm{d} \mu)}^2+2\|\pi_B f-f_\rho\|_{L^2(\mathrm{d} \mu)}^2\nonumber\\
    \le&2\|\pi_{2B}(f_m-f)\|_{L^2(\mathrm{d} \mu)}^2+2\|f-f_\rho\|_{L^2(\mathrm{d} \mu)}^2\,,
\end{align*}
where we have used $|\pi_B f_m-\pi_B f|\le|\pi_{2B}(f_m-f)|$ and $B>\sup_{x\in\Omega}|f_\rho(x)|$. The Sobolev approximation bound recalled in Section~\ref{sec:approximation} controls the second term. For the first term, we follow \cite[Chapters 13--14]{Wainwright_2019}: first estimate the empirical error $\|\pi_{2B}(f_m-f)\|_{L^2(\mathrm{d} \mu_n)}^2$ under a fixed-design setting, and then pass to the random-design setting through a uniform law of large numbers. This yields an upper bound on $\|\pi_{2B}(f_m-f)\|_{L^2(\mathrm{d} \mu)}^2$. For the empirical error $\|\pi_{2B}(f_m-f)\|_{L^2(\mathrm{d} \mu_n)}^2$, note that for any $f\in\Sigma_m^s$,
    \begin{align*}
    \|\pi_{2B}(f_m-f)\|_{L^2(\mathrm{d} \mu_n)}^2\le\|f_m-f\|_{L^2(\mathrm{d} \mu_n)}^2\le& 2\|f_m-f_\rho\|_{L^2(\mathrm{d} \mu_n)}^2+2\|f-f_\rho\|_{L^2(\mathrm{d} \mu_n)}^2\,,
\end{align*}
The same Sobolev approximation bound controls the second term on the right-hand side. It remains to estimate the empirical error $\|f_m-f_\rho\|_{L^2(\mathrm{d} \mu_n)}^2$ of the empirical risk minimizer $f_m$. The proof proceeds in three steps.

\subsubsection*{Step 1: Upper bounds on the empirical error and the $\ell^1$ path norm}
In this subsection, we fix the sample $\{(\xb_i,y_i)\}_{i=1}^n$ and analyze upper bounds for the empirical error and the $\ell^1$ path norm of the empirical risk minimizer $f_m$. For any $f\in\Sigma_m^s$, by definition,
\begin{align*}
    \frac{1}{n}\sum_{i=1}^n(f_m(\xb_i)-y_i)^2+\lambda\|f_m\|_{\pscr_{1,s}}\le \frac{1}{n}\sum_{i=1}^n(f(\xb_i)-y_i)^2+\lambda\|f\|_{\pscr_{1,s}}\,.
\end{align*}
Noting that $y_i=f_\rho(\xb_i)+\eta_i$, substituting this and rearranging yields the \emph{basic inequality}
\begin{align*}
    \|f_m-f_\rho\|_{L^2(\mathrm{d} \mu_n)}^2+\lambda\|f_m\|_{\pscr_{1,s}}\le& \|f-f_\rho\|_{L^2(\mathrm{d} \mu_n)}^2+\lambda\|f\|_{\pscr_{1,s}}+\frac{2}{n}\sum_{i=1}^n\eta_i(f_m(\xb_i)-f(\xb_i))\,.
\end{align*}
From this basic inequality and the oracle inequality to be derived below, we can obtain upper bounds for both the empirical error and the $\ell^1$ path norm. To state and prove this oracle inequality, we need a tail bound for the empirical process. Conditional on the design $\{\xb_i\}_{i=1}^n$, Assumption~\ref{assumption:noise} gives, for all $f,g\in T(\delta)$,
\begin{align*}
    \left\|\frac{1}{n}\sum_{i=1}^n\eta_i\bigl(f(\xb_i)-g(\xb_i)\bigr)\right\|_{\psi_2}
    \lesssim \frac{\sigma}{\sqrt n}\|f-g\|_{L^2(\mathrm{d}\mu_n)}\,.
\end{align*}
Thus the centered empirical process satisfies the $\psi_2$ increment condition required in \cite[Theorem~3.2]{Dirksen15}; applying that result to the symmetrized index class gives the following bound.
\begin{theorem}[Generic chaining tail bound]\label{thm:concentration}
    Let $\delta>0$ and define the sub-Gaussian empirical process
    \begin{align*}
        Z(\delta)=\sup_{\substack{f\in\starhull(\partial(\Sigma_{m}^s)),\\\|f\|_{L^2(\mathrm{d} \mu_n)}\le \delta}}\left|\frac{1}{n}\sum_{i=1}^n\eta_if(\xb_i)\right|\,,
    \end{align*}
    Then there exist constants $L,L'>0$ depending only on the sub-Gaussian parameter $\sigma$ such that $Z$ satisfies the concentration inequality
    \begin{align*}
        \pbb\left[Z(\delta)\ge \frac{L}{\sqrt{n}}\gamma_2(\delta)+t\right]\le L'e^{-\frac{nt^2}{4L^2\delta^2}}\ \ \forall t>0\,.
    \end{align*}
\end{theorem}

\cite{yang2024nonparametricregressionusingoverparameterized} established generalization results under the $\ell_2$ path norm only for bounded noise. For bounded noise or Gaussian noise, one may directly use the classical Hoeffding inequality or Gaussian concentration inequality to obtain similar conclusions. Here we provide the proof under sub-Gaussian noise. In this case, the previous two concentration inequalities are no longer directly applicable, so we use the finer quantity $\gamma_2(\delta)$ to characterize the concentration property of the empirical process. 
\begin{lemma}\label{lem:oracle}
    Let $c_0\ge1$ be a constant, and suppose $f_m$ is the empirical risk minimizer such that for any $f\in\Sigma_m^s$,
    \begin{align*}
        \|f_m-f_\rho\|_{L^2(\mathrm{d} \mu_n)}^2+\lambda\|f_m\|_{\pscr_{1,s}}\le& c_0\left(\|f-f_\rho\|_{L^2(\mathrm{d} \mu_n)}^2+\lambda\|f\|_{\pscr_{1,s}}\right)+\frac{2}{n}\left|\sum_{i=1}^n\eta_i(f_m(\xb_i)-f(\xb_i))\right|\,.
    \end{align*}
    For a user-defined parameter $M>0$, let $\delta_n=\delta_n(M)$ satisfy
    \begin{align*}
        \frac{L}{\sqrt{n}}\gamma_2(\delta_n)\le\delta_n^2\,.
    \end{align*}
    Then there exist constants $c_1,L'>0$ such that if $\lambda\ge\frac{8\delta_n^2}{M}$, then
    \begin{align}\label{eq:oracle}
        \|f_m-f_\rho\|_{L^2(\mathrm{d} \mu_n)}^2+\lambda\|f_m\|_{\pscr_{1,s}}\le(1+2c_0)\inf_{f\in\Sigma_{m}^s}\left(\|f-f_\rho\|_{L^2(\mathrm{d} \mu_n)}^2+\lambda\|f\|_{\pscr_{1,s}}\right)+64\delta_n^2\,,
    \end{align}
    with probability at least $1-L'e^{-\frac{n\delta_n^2}{4L^2}}$, where the implied constants are independent of the sample.
\end{lemma}

\begin{proof}
    For any $f\in\Sigma_m^s$, write $g_f=f_m-f$. Then $\|g_f\|_{\pscr_{1,s}}\le\|f_m\|_{\pscr_{1,s}}+\|f\|_{\pscr_{1,s}}$. We divide the discussion into four cases.

    \textbf{Case 1}: $\|g_f\|_{L^2(\mathrm{d} \mu_n)}^2+\lambda\|g_f\|_{\pscr_{1,s}}\le\delta_n^2$. Then
    \begin{align*}
        \|f_m-f_\rho\|_{L^2(\mathrm{d} \mu_n)}^2+\lambda\|f_m\|_{\pscr_{1,s}}\le& 2\|f-f_\rho\|_{L^2(\mathrm{d} \mu_n)}^2+2\|g_f\|_{L^2(\mathrm{d} \mu_n)}^2+\lambda\|f\|_{\pscr_{1,s}}+\lambda\|g_f\|_{\pscr_{1,s}}\\
        \le&2\|f-f_\rho\|_{L^2(\mathrm{d} \mu_n)}^2+\lambda\|f\|_{\pscr_{1,s}}+2\delta_n^2\,.
    \end{align*}

    \textbf{Case 2}: $\|g_f\|_{L^2(\mathrm{d} \mu_n)}^2+\lambda\|g_f\|_{\pscr_{1,s}}>\delta_n^2$ and $\|f_m\|_{\pscr_{1,s}}\le M$. In this case $g_f\in\partial\Sigma_m^s$. Consider the event
    \begin{align*}
        A_1=\left\{\exists g\in\partial\Sigma_m^s:\|g\|_{L^2(\mathrm{d} \mu_n)}^2+\lambda\|g\|_{\pscr_{1,s}}>\delta_n^2,\left|\frac{1}{n}\sum_{i=1}^n\eta_ig(\xb_i)\right|>2\delta_n\sqrt{\|g\|_{L^2(\mathrm{d} \mu_n)}^2+\lambda\|g\|_{\pscr_{1,s}}}\right\}\,.
    \end{align*}
    If $A_1$ occurs, then there exists
    \begin{align*}
        \tilde{g}=\frac{\delta_n}{\sqrt{\|g\|_{L^2(\mathrm{d} \mu_n)}^2+\lambda\|g\|_{\pscr_{1,s}}}}g\in\starhull(\partial\Sigma_m^s)
    \end{align*}
    such that $\|\tilde{g}\|_{L^2(\mathrm{d} \mu_n)}\le\delta_n$ and
    \begin{align*}
        \left|\frac{1}{n}\sum_{i=1}^n\eta_i\tilde{g}(\xb_i)\right|=&\frac{\delta_n}{\sqrt{\|g\|_{L^2(\mathrm{d} \mu_n)}^2+\lambda\|g\|_{\pscr_{1,s}}}}\left|\frac{1}{n}\sum_{i=1}^n\eta_ig(\xb_i)\right|>2\delta_n^2\,.
    \end{align*}
    This implies that $\pbb[A_1]\le\pbb[Z_n(\delta_n)>2\delta_n^2]$, where
    \begin{align*}
        Z_n(\delta_n)=\sup_{\substack{g\in\starhull(\partial\Sigma_m^s),\\\|g\|_{L^2(\mathrm{d} \mu_n)}\le\delta_n}}\left|\frac{1}{n}\sum_{i=1}^n\eta_ig(\xb_i)\right|\,.
    \end{align*}
    is a sub-Gaussian empirical process. By Theorem~\ref{thm:concentration},
    \begin{align*}
        \pbb\left[Z_n(\delta_n)\ge \frac{L}{\sqrt{n}}\gamma_2(\delta_n)+t\right]\le L'e^{-\frac{nt^2}{4L^2\delta_n^2}}\,.
    \end{align*}
    Since by assumption $\frac{L}{\sqrt{n}}\gamma_2(\delta_n)\le \delta_n^2$, taking $t=\delta_n^2$ yields
    \begin{align*}
        \pbb[A_1]\le\pbb[Z_n(\delta_n)>2\delta_n^2]\le \pbb\left[Z_n(\delta_n)\ge \frac{L}{\sqrt n}\gamma_2(\delta_n)+\delta_n^2\right]\le L'e^{-\frac{n\delta_n^2}{4L^2}}\,.
    \end{align*}
    This shows that $\overline{A}_1$ holds with high probability; specifically,
    \begin{align*}
        \pbb\left[\left|\frac{1}{n}\sum_{i=1}^n\eta_ig_f(\xb_i)\right|\le 2\delta_n\sqrt{\|g_f\|_{L^2(\mathrm{d} \mu_n)}^2+\lambda\|g_f\|_{\pscr_{1,s}}}\right]\ge 1-L'e^{-\frac{n\delta_n^2}{4L^2}}\,.
    \end{align*}
    Combining this with the assumed inequality, it follows that with at least the same probability,
    \begin{align*}
        &\|f_m-f_\rho\|_{L^2(\mathrm{d} \mu_n)}^2+\lambda\|f_m\|_{\pscr_{1,s}}\\
        \le& c_0\left(\|f-f_\rho\|_{L^2(\mathrm{d} \mu_n)}^2+\lambda\|f\|_{\pscr_{1,s}}\right)+4\delta_n\sqrt{\|g_f\|_{L^2(\mathrm{d} \mu_n)}^2+\lambda\|g_f\|_{\pscr_{1,s}}}\\
        \le & c_0\left(\|f-f_\rho\|_{L^2(\mathrm{d} \mu_n)}^2+\lambda\|f\|_{\pscr_{1,s}}\right)+16\delta_n^2+\frac{1}{4}\|f_m-f\|_{L^2(\mathrm{d} \mu_n)}^2+\frac{1}{4}\lambda\|f_m-f\|_{\pscr_{1,s}}\\
        \le &c_0\left(\|f-f_\rho\|_{L^2(\mathrm{d} \mu_n)}^2+\lambda\|f\|_{\pscr_{1,s}}\right)+16\delta_n^2+\frac{1}{2}\|f_m-f_\rho\|_{L^2(\mathrm{d} \mu_n)}^2+\frac{1}{2}\|f-f_\rho\|_{L^2(\mathrm{d} \mu_n)}^2\\
        &+\frac{1}{4}\lambda\|f_m\|_{\pscr_{1,s}}+\frac{1}{4}\lambda\|f\|_{\pscr_{1,s}}\\
        \le&\left(\frac{1}{2}+c_0\right)\left(\|f-f_\rho\|_{L^2(\mathrm{d} \mu_n)}^2+\lambda\|f\|_{\pscr_{1,s}}\right)+16\delta_n^2+\frac{1}{2}\left(\|f_m-f_\rho\|_{L^2(\mathrm{d} \mu_n)}^2+\lambda\|f_m\|_{\pscr_{1,s}}\right)\,,
    \end{align*}
    or equivalently,
    \begin{align*}
        \|f_m-f_\rho\|_{L^2(\mathrm{d} \mu_n)}^2+\lambda\|f_m\|_{\pscr_{1,s}}\le(1+2c_0)\left(\|f-f_\rho\|_{L^2(\mathrm{d} \mu_n)}^2+\lambda\|f\|_{\pscr_{1,s}}\right)+32\delta_n^2\,.
    \end{align*}

    \textbf{Case 3}: $\|f_m\|_{\pscr_{1,s}}>M$ and $\|g_f\|_{L^2(\mathrm{d} \mu_n)}\le\frac{\delta_n\|f_m\|_{\pscr_{1,s}}}{M}$. In this case,
    \begin{align*}
        &\tilde{g}_f=\frac{M}{\|f_m\|_{\pscr_{1,s}}}g_f=\frac{M}{\|f_m\|_{\pscr_{1,s}}}(f_m-f)\in\partial\Sigma_m^s\,,\\
        &\|\tilde{g}_f\|_{L^2(\mathrm{d} \mu_n)}=\frac{M}{\|f_m\|_{\pscr_{1,s}}}\|g_f\|_{L^2(\mathrm{d} \mu_n)}\le \delta_n\,.
    \end{align*}
    By Case 2, $\pbb[Z_n(\delta_n)>2\delta_n^2]\le L'e^{-\frac{n\delta_n^2}{4L^2}}.$ Therefore,
    \begin{align*}
        \left|\frac{1}{n}\sum_{i=1}^n\eta_ig_f(\xb_i)\right|=\frac{\|f_m\|_{\pscr_{1,s}}}{M}\left|\frac{1}{n}\sum_{i=1}^{n}\eta_i\tilde{g}_f(\xb_i)\right|\le \frac{2\delta_n^2\|f_m\|_{\pscr_{1,s}}}{M}\,,
    \end{align*}
    holds with probability at least $1-L'e^{-\frac{n\delta_n^2}{4L^2}}$. Combining this with the basic inequality, we find that with at least the same probability,
    \begin{align*}
        &\|f_m-f_\rho\|_{L^2(\mathrm{d} \mu_n)}^2+\lambda\|f_m\|_{\pscr_{1,s}}\\
        \le& c_0\left(\|f-f_\rho\|_{L^2(\mathrm{d} \mu_n)}^2+\lambda\|f\|_{\pscr_{1,s}}\right)+\frac{4\delta_n^2}{M}\|f_m\|_{\pscr_{1,s}}\\
        \le& c_0\left(\|f-f_\rho\|_{L^2(\mathrm{d} \mu_n)}^2+\lambda\|f\|_{\pscr_{1,s}}\right)+\frac{\lambda}{2}\|f_m\|_{\pscr_{1,s}}\,,
    \end{align*}
    where the last step uses $\lambda\ge \frac{8\delta_n^2}{M}$. Hence,
    \begin{align*}
        \|f_m-f_\rho\|_{L^2(\mathrm{d} \mu_n)}^2+\lambda\|f_m\|_{\pscr_{1,s}}\le& 2c_0\left(\|f-f_\rho\|_{L^2(\mathrm{d} \mu_n)}^2+\lambda\|f\|_{\pscr_{1,s}}\right)\,.
    \end{align*}

    \textbf{Case 4}: $\|f_m\|_{\pscr_{1,s}}>M$ and $\|g_f\|_{L^2(\mathrm{d} \mu_n)}>\frac{\delta_n\|f_m\|_{\pscr_{1,s}}}{M}$. Under this assumption,
    \begin{align*}
        \tilde{g}_f=\frac{M}{\|f_m\|_{\pscr_{1,s}}}g_f\in\partial\Sigma_m^s,\ \ \|\tilde{g}_f\|_{L^2(\mathrm{d} \mu_n)}>\delta_n\,.
    \end{align*}
    We show that the event
    \begin{align*}
        A_2=\left\{\exists g\in\partial\Sigma_m^s:\|g\|_{L^2(\mathrm{d} \mu_n)}>\delta_n,\left|\frac{1}{n}\sum_{i=1}^n\eta_ig(\xb_i)\right|>2\delta_n\|g\|_{L^2(\mathrm{d} \mu_n)}+\frac{1}{16}\|g\|_{L^2(\mathrm{d} \mu_n)}^2\right\}
    \end{align*}
    occurs with small probability. To this end, for $k=0,1,\cdots$ define $t_k=2^k\delta_n$ and the event
    \begin{align*}
        A_2^k=\left\{\exists g\in\partial\Sigma_m^s:t_k<\|g\|_{L^2(\mathrm{d} \mu_n)}\le 2t_k,\left|\frac{1}{n}\sum_{i=1}^n\eta_ig(\xb_i)\right|>2\delta_n\|g\|_{L^2(\mathrm{d} \mu_n)}+\frac{1}{16}\|g\|_{L^2(\mathrm{d} \mu_n)}^2\right\}\,.
    \end{align*}
    Then $A_2=\bigcup_{k=0}^\infty A_2^k$. If $A_2^k$ occurs, then there exists $g\in\partial\Sigma_m^s$ such that $t_k<\|g\|_{L^2(\mathrm{d} \mu_n)}\le t_{k+1}$ and
    \begin{align*}
        \left|\frac{1}{n}\sum_{i=1}^n\eta_ig(\xb_i)\right|>&2\delta_n\|g\|_{L^2(\mathrm{d} \mu_n)}+\frac{1}{16}\|g\|_{L^2(\mathrm{d} \mu_n)}^2\\
        \ge&2\delta_nt_k+\frac{1}{16}t_k^2=t_{k+1}\delta_n+\frac{1}{64}t_{k+1}^2\,,
    \end{align*}
    and therefore
    \begin{align*}
        \pbb[A_2^k]\le\pbb\left[Z_n(t_{k+1})>t_{k+1}\delta_n+\frac{1}{64}t_{k+1}^2\right]\,.
    \end{align*}
    By Theorem~\ref{thm:concentration},
    \begin{align*}
        \pbb\left[Z_n(t_{k+1})\ge \frac{L}{\sqrt{n}}\gamma_2(t_{k+1})+t\right]\le L'e^{-\frac{nt^2}{4L^2t_{k+1}^2}}\,.
    \end{align*}
    Since $\starhull(\partial\Sigma_m^s)$ is star-shaped, $\frac{\gamma_2(t)}{t}$ is decreasing in $t$, and thus for any $t\ge\delta_n$,
    \begin{align*}
        \frac{L\gamma_2(t)}{\sqrt{n}}\le t\frac{L\gamma_2(\delta_n)}{\delta_n\sqrt{n}}\le t\delta_n\,,
    \end{align*}
    hence taking the deviation parameter in the preceding display to be $\frac{1}{64}t^2$ yields
    \begin{align*}
        \pbb\left[Z_n(t)>t\delta_n+\frac{1}{64}t^2\right]\le L'e^{-\frac{nt^2}{4^7L^2}}\,.
    \end{align*}
    By countable additivity of probability,
    \begin{align*}
        \pbb[A_2]\le\sum_{k=0}^\infty\pbb[A_2^k]\le\sum_{k=0}^\infty L'e^{-\frac{nt_k^2}{4^7L^2}}=L'\sum_{k=0}^\infty e^{-\frac{n\delta_n^2}{4^7L^2}4^k}\le c_1e^{-\frac{n\delta_n^2}{4^7L^2}}\,.
    \end{align*}
    This shows that, with the same probability,
    \begin{align*}
        \left|\frac{1}{n}\sum_{i=1}^n\eta_i\tilde g_f(\xb_i)\right|\le 2\delta_n\|\tilde g_f\|_{L^2(\mathrm{d} \mu_n)}+\frac{1}{16}\|\tilde g_f\|_{L^2(\mathrm{d} \mu_n)}^2\,.
    \end{align*}
    Multiplying both sides by $\frac{\|f_m\|_{\pscr_{1,s}}}{M}$ and using $\tilde g_f=\frac{M}{\|f_m\|_{\pscr_{1,s}}}g_f$, we get
    \begin{align*}
        \left|\frac{1}{n}\sum_{i=1}^n\eta_ig_f(\xb_i)\right|\le&2\delta_n\|g_f\|_{L^2(\mathrm{d} \mu_n)}+\frac{1}{16}\frac{M}{\|f_m\|_{\pscr_{1,s}}}\|g_f\|_{L^2(\mathrm{d} \mu_n)}^2\\
        \le&2\delta_n\|g_f\|_{L^2(\mathrm{d} \mu_n)}+\frac{1}{16}\|g_f\|_{L^2(\mathrm{d} \mu_n)}^2\,,
    \end{align*}
    where the last step uses the assumption $\|f_m\|_{\pscr_{1,s}}>M$. Combining this with the basic inequality, we find that with at least the same probability,
    \begin{align*}
        &\|f_m-f_\rho\|_{L^2(\mathrm{d} \mu_n)}^2+\lambda\|f_m\|_{\pscr_{1,s}}\\
        \le& c_0\left(\|f-f_\rho\|_{L^2(\mathrm{d} \mu_n)}^2+\lambda\|f\|_{\pscr_{1,s}}\right)+4\delta_n\|g_f\|_{L^2(\mathrm{d} \mu_n)}+\frac{1}{8}\|g_f\|_{L^2(\mathrm{d} \mu_n)}^2\\
        \le & c_0\left(\|f-f_\rho\|_{L^2(\mathrm{d} \mu_n)}^2+\lambda\|f\|_{\pscr_{1,s}}\right)+32\delta_n^2+\frac{1}{4}\|f_m-f_\rho\|_{L^2(\mathrm{d} \mu_n)}^2+\frac{1}{4}\|f-f_\rho\|_{L^2(\mathrm{d} \mu_n)}^2\\
        \le & c_0\left(\|f-f_\rho\|_{L^2(\mathrm{d} \mu_n)}^2+\lambda\|f\|_{\pscr_{1,s}}\right)+32\delta_n^2+\frac{1}{2}\|f_m-f_\rho\|_{L^2(\mathrm{d} \mu_n)}^2+\frac{1}{2}\|f-f_\rho\|_{L^2(\mathrm{d} \mu_n)}^2\\
        \le&\left(\frac{1}{2}+c_0\right)\left(\|f-f_\rho\|_{L^2(\mathrm{d} \mu_n)}^2+\lambda\|f\|_{\pscr_{1,s}}\right)+32\delta_n^2+\frac{1}{2}\left(\|f_m-f_\rho\|_{L^2(\mathrm{d} \mu_n)}^2+\lambda\|f_m\|_{\pscr_{1,s}}\right)\,,
    \end{align*}
    and rearranging gives
    \begin{align*}
        \|f_m-f_\rho\|_{L^2(\mathrm{d} \mu_n)}^2+\lambda\|f_m\|_{\pscr_{1,s}}\le(1+2c_0)\left(\|f-f_\rho\|_{L^2(\mathrm{d} \mu_n)}^2+\lambda\|f\|_{\pscr_{1,s}}\right)+64\delta_n^2\,.
    \end{align*}
    This completes the proof of Lemma~\ref{lem:oracle}.
\end{proof}

Lemma~\ref{lem:oracle} is in fact an oracle inequality. Equation~\eqref{eq:oracle} shows that the empirical error of the empirical risk minimizer $f_m$ can be decomposed into two parts: one is the approximation error
$\inf_{f\in\Sigma_m^s}\left(\|f-f_\rho\|_{L^2(\mathrm{d} \mu_n)}^2+\lambda\|f\|_{\pscr_{1,s}}\right)$,
and the other is the complexity term $64\delta_n^2$ of the model class. The first term is controlled by suitable approximation bounds for the target function class, whereas the second is handled by tools from statistical learning theory.

Equation~\eqref{eq:oracle} also provides an upper bound for the $\ell^1$ path norm. As an example, suppose that $Y$ has a uniform bound $B>0$ (in which case ${\etab}$ is bounded and hence sub-Gaussian). In general, taking $f=0$ in the basic inequality yields
\begin{align*}
    \|f_m\|_{\pscr_{1,s}}\le \frac{1}{n\lambda}\sum_{i=1}^n y_i^2\le \frac{B^2}{\lambda}\,,
\end{align*}
but this estimate is quite crude. In fact, if we define
\begin{align*}
    f^*=\arg\min_{f\in\Sigma_n^s}\|f-f_\rho\|_{L^2(\mathrm{d} \mu)}^2+\lambda\|f\|_{\pscr_{1,s}}\,,
\end{align*}
then by definition,
$$\|f^*\|_{\pscr_{1,s}}\le\frac{\inf_{f\in\Sigma_n^s}\left(\|f-f_\rho\|_{L^2(\mathrm{d} \mu)}^2+\lambda\|f\|_{\pscr_{1,s}}\right)}{\lambda}\,.$$
Since $f_m$ is an approximation of $f^*$, we expect $\|f_m\|_{\pscr_{1,s}}$ to admit a similar upper bound, and this upper bound is much smaller than $\frac{B^2}{\lambda}$. Indeed, if $\lambda$ is chosen appropriately (for example, so that $\lim_{m\to\infty}\lambda(m)=0$), then
$\inf_{f\in\Sigma_n^s}\left(\|f-f_\rho\|_{L^2(\mathrm{d} \mu)}^2+\lambda\|f\|_{\pscr_{1,s}}\right)\to 0$
while $B^2$ is a fixed constant. 
This phenomenon has been studied in \citep{steinwart2008support,ShiFengZhou2011,leishi19}, where iterative methods were proposed to improve this upper bound. However, using the oracle inequality above, we can directly obtain a better upper bound. Specifically, if we take $\lambda=\frac{8\delta_n^2}{M}$, then with high probability
\begin{align*}
    \|f_m\|_{\pscr_{1,s}}\lesssim \frac{\inf_{f\in\Sigma_m^s}\|f-f_\rho\|_{L^\infty(\Omega)}^2}{\lambda}+M\,,
\end{align*}
which is already very close to
$\frac{\inf_{f\in\Sigma_n^s}\left(\|f-f_\rho\|_{L^2(\mathrm{d} \mu)}^2+\lambda\|f\|_{\pscr_{1,s}}\right)}{\lambda}$.

\subsubsection*{Step 2: Localization and the uniform law of large numbers}
In the previous subsection, we treated the sample $\{(\xb_i,y_i)\}_{i=1}^n$ as fixed and obtained an upper bound on the empirical error. However, what we are actually concerned with is the generalization error. In this subsection, we regard the sample as random and use the localization technique and uniform law of large numbers introduced in \cite[Chapter 14]{Wainwright_2019} to analyze the upper bound on the generalization error. 

Since the $\xb_i$ are independent and identically distributed, for any square-integrable $f$ we have
\begin{align*}
    \ebb[\|f\|_{L^2(\mathrm{d} \mu_n)}^2]=\ebb\left[\frac{1}{n}\sum_{i=1}^n f(\xb_i)^2\right]=\ebb[f(\xb)^2]=\|f\|_{L^2(\mathrm{d} \mu)}^2\,.
\end{align*}
Therefore, by the law of large numbers, under suitable conditions $\|f\|_{L^2(\mathrm{d} \mu_n)}$ converges to $\|f\|_{L^2(\mathrm{d} \mu)}$. For example, if $f$ is bounded, then Hoeffding's inequality \cite[Proposition 2.5]{Wainwright_2019} yields
\begin{align*}
    \pbb\left[\left|\|f\|_{L^2(\mathrm{d} \mu_n)}^2-\|f\|_{L^2(\mathrm{d} \mu)}^2\right|\ge t\right]\le 2e^{-cnt^2},\ \ \forall t>0\,.
\end{align*}
Using the uniform law of large numbers, one can make this tail bound hold uniformly over function classes satisfying certain conditions. The following theorem is taken from \cite[Theorem 14.1, Proposition 14.25]{Wainwright_2019} or \cite[Lemma 7]{yang2024nonparametricregressionusingoverparameterized}.
\begin{theorem}\label{thm:uniformlawoflarge}
    Let the function class $\fcal$ be star-shaped and uniformly bounded by some constant $B>0$. Let $\varepsilon_n$ satisfy
    \begin{align*}
        \gcal_n(\fcal,\varepsilon_n,\varepsilon)\le\frac{\varepsilon_n^2}{B}\,,
    \end{align*}
    where $\varepsilon$ denotes independent identically distributed Rademacher random variables. Then there exist constants $c_1,c_2>0$ such that
    \begin{align*}
        \|f\|_{L^2(\mathrm{d} \mu)}^2\le 2\|f\|_{L^2(\mathrm{d} \mu_n)}^2+\varepsilon_n^2,\ \ \forall f\in\fcal\,,
    \end{align*}
    with probability at least $1-c_1e^{-\frac{c_2n\varepsilon_n^2}{B^2}}$.
\end{theorem}

Since for any $f\in\Sigma_{m,M}^s$ we have
\begin{align*}
    \pi_{2B}(f_m-f)\in\starhull(\pi_{2B}(\Sigma_{2m,\|f_m\|_{\pscr_{1,s}}+M}^s))\subset\starhull(\pi_{2B}(\bscr_s(\|f_m\|_{\pscr_{1,s}}+M)))\,,
\end{align*}
the theorem above implies that
\begin{align*}
    \|\pi_{2B}(f_m-f)\|_{L^2(\mathrm{d} \mu)}^2\le 2\|\pi_{2B}(f_m-f)\|_{L^2(\mathrm{d} \mu_n)}^2+\varepsilon_n^2,\ \ \forall f\in\Sigma_{m,M}^s\,,
\end{align*}
holds with high probability.
\subsubsection*{Step 3: Combining the above analysis to complete the upper bound for the generalization error}
\par\noindent{\bf Proof of Theorem~\ref{thm:upperbound}.}
    We first determine $\delta_n=\delta_n(M)$ satisfying $\frac{L\gamma_2(\delta_n)}{\sqrt{n}}\le\delta_n^2$, where $M>0$ is to be chosen. Without loss of generality, assume that $n$ is sufficiently large so that $\delta_n\le M$. Since $\starhull(\partial\Sigma_m^s)\subset\bscr_s(2M)$, we have
    \begin{align*}
        \frac{L\gamma_2(\delta_n)}{\sqrt{n}}\lesssim&\frac{1}{\sqrt{n}}\int_0^{2\delta_n}\sqrt{\log\mathcal{N}(u,\bscr_s(2M),\|\cdot\|_{L^2(\mathrm{d} \mu_n)})}du\\
        \le&\frac{1}{\sqrt{n}}\int_0^{2\delta_n}\sqrt{\log\mathcal{N}(u,\bscr_s(2M),\|\cdot\|_{L^2(\Omega)})}du\\
        \lesssim&\frac{1}{\sqrt{n}}\int_0^{\delta_n}\left(\frac{u}{M}\right)^{-\frac{d}{d+2s+1}}\sqrt{\log\left(1+\frac{M}{u}\right)}du\\
        \lesssim&\frac{M^{\frac{d}{d+2s+1}}}{\sqrt{n}}\int_0^{\delta_n}u^{-\frac{d}{d+2s+1}}\left(\sqrt{\log2M}+\sqrt{\log\frac{1}{u}}\right)du\\
        \lesssim&\delta_n^{\frac{2s+1}{d+2s+1}}M^{\frac{d}{d+2s+1}}n^{-\frac{1}{2}}\sqrt{\log\frac{nM}{\delta_n}}\,.
    \end{align*}
    Thus we may take
    \begin{align*}
    \delta_n^2\asymp n^{-\frac{d+2s+1}{2d+2s+1}}M^{\frac{2d}{2d+2s+1}}\log nM\,,
\end{align*}
    and denote the approximation error of the path-norm constrained class by
    \begin{align*}
    \escr(M)=\inf_{f\in\Sigma_{m,M}^s}\|f-f_\rho\|_{L^\infty(\Omega)}^2\,.
\end{align*}
    By Lemma~\ref{lem:oracle}, taking $\lambda=\frac{8\delta_n^2}{M}$ yields
    \begin{align*}
    \|f_m-f_\rho\|_{L^2(\mathrm{d} \mu_n)}^2+\lambda\|f_m\|_{\pscr_{1,s}}\lesssim \escr(M)+\delta_n^2\,,
\end{align*}
    and in particular,
    \begin{align*}
    \|f_m\|_{\pscr_{1,s}}\lesssim\frac{\escr(M)+\delta_n^2}{\lambda}\lesssim\frac{\escr(M)}{\lambda}+M\,.
\end{align*}
    By the approximation result of \citep{MaoSiegelXu2026,Siegel_2022sharp}, there exists a $\bar f_m\in\Sigma_{m}^s$ such that
    \begin{align*}
    \|\bar f_m-f_\rho\|_{L^\infty(\Omega)}\lesssim \begin{cases}
        m^{-\frac{1}{2}-\frac{2s+1}{2d}}\,,&\alpha\ge s+\frac{d+1}{2}\\
        m^{-\frac{\alpha}{d}}\,,&\alpha\in\nbb\cap\left(0,s+\frac{d+1}{2}\right)
    \end{cases}\,,
\end{align*}
    but in order to estimate $\escr(M)$, we also need to estimate the $\ell^1$ path norm $\|\bar f_m\|_{\pscr_{1,s}}$ of the network $\bar f_m$. From the proof of \citep{MaoSiegelXu2026,Siegel_2022sharp}, we can choose
    $\bar f_m(\xb)=\frac{1}{m}\sum_{i=1}^m ma_i\sigma_s(\wb_i\cdot \xb+b_i)$
    satisfying $\sum_{i=1}^m|a_i|\le \widetilde{M},\wb_i\in\sbb^{d-1},b_i\in[-\sqrt{d},\sqrt{d}]$, where
    \begin{align*}
        \widetilde{M}=\begin{cases}
            Cm^{\frac{2s+d+1-2\alpha}{2d}}\,,&\alpha<s+\frac{d+1}{2}\text{ and }\alpha\in\nbb\\
            1\,,&\alpha\ge s+\frac{d+1}{2}
        \end{cases}\,,
    \end{align*}
    and therefore the $\ell^1$ path norm of such a $\bar f_m$ satisfies
    \begin{align*}
        \|\bar f_m\|_{\pscr_{1,s}}\le 2^sd^{\frac{s}{2}}\widetilde{M}\,,
    \end{align*}
    which means that as long as we take $M\ge2^sd^{\frac{s}{2}}\widetilde{M}$, we have
    \begin{align*}
        \escr(M)\le\|\bar f_m-f_\rho\|_{L^\infty(\Omega)}^2\lesssim \begin{cases}
            m^{-\frac{d+2s+1}{d}}\,,&\alpha\ge s+\frac{d+1}{2}\\
            m^{-\frac{2\alpha}{d}}\,,&\alpha\in\nbb\cap\left(0,s+\frac{d+1}{2}\right)
        \end{cases}\,.
    \end{align*}
    If $\alpha\ge s+\frac{d+1}{2}$, taking
$$M=2^sd^{\frac{s}{2}}\asymp 1,\ \ m\gtrsim n^{\frac{d}{2d+2s+1}}$$
    yields
    \begin{align*}
    \escr(M)\lesssim n^{-\frac{d+2s+1}{2d+2s+1}},\ \ \delta_n^2\asymp n^{-\frac{d+2s+1}{2d+2s+1}}\log n,\ \ \lambda=\frac{8\delta_n^2}{M}\asymp n^{-\frac{d+2s+1}{2d+2s+1}}\log n\,.
\end{align*}
    If $\alpha\in\nbb\cap(0,s+\frac{d+1}{2})$, taking
$$M=2^sd^{\frac{s}{2}}m^{\frac{2s+d+1-2\alpha}{2d}},\ \ m\gtrsim n^{\frac{d}{2\alpha+d}}$$
    yields
    \begin{align}
    \escr(M)\lesssim n^{-\frac{2\alpha}{2\alpha+d}},\ \ \delta_n^2\asymp n^{-\frac{2\alpha}{2\alpha+d}}\log n,\ \ \lambda=\frac{8\delta_n^2}{M}\asymp n^{-\frac{1}{2}-\frac{2s+1}{4\alpha+2d}}\log n\,.\label{eq:parameters2}
\end{align}
    In both cases, we have $\escr(M)\lesssim \lambda M$, hence $\|f_m\|_{\pscr_{1,s}}\lesssim M$. Taking $\fcal=\starhull(\pi_{2B}(\bscr_s(\|f_m\|_{\pscr_{1,s}}+M)))$ in Theorem~\ref{thm:uniformlawoflarge}, we obtain
    \begin{align*}
        \|\pi_{2B}(f_m-f)\|_{L^2(\mathrm{d} \mu)}^2\lesssim \|\pi_{2B}(f_m-f)\|_{L^2(\mathrm{d} \mu_n)}^2+\varepsilon_n^2\ \ \forall f\in\Sigma_{m,M}^s\,,
    \end{align*}
    with probability at least $1-c_1e^{-\frac{c_2n\varepsilon_n^2}{B^2}}$, where
    \begin{align*}
        \varepsilon_n^2\asymp n^{-\frac{d+2s+1}{2d+2s+1}}M^{\frac{2d}{2d+2s+1}}\log nM\asymp\delta_n^2\,.
    \end{align*}
    Combining everything above and taking $f=\bar f_m$, we obtain
    \begin{align*}
        \|\pi_B f_m-f_\rho\|_{L^2(\mathrm{d} \mu)}^2\lesssim& \escr(M)+\delta_n^2+\varepsilon_n^2\,,
    \end{align*}
    with probability at least $1-c_1e^{-c_2n\delta_n^2}$. The desired conclusion follows from \eqref{eq:parameters1} and \eqref{eq:parameters2}. This completes the proof of Theorem~\ref{thm:upperbound}.
\hfill\BlackBox\\[2mm]

\subsection*{Proofs of Theorems~\ref{thm:lowerbound1} and~\ref{thm:lowerbound2}}
To estimate the minimax lower bound, we first need some preliminaries. For each probability distribution $\rho$ on $\Omega\times\rbb$, we can define a regression function
\begin{align*}
    f_\rho(x)=\int_{\rbb}y \mathrm{d}\rho(y|x)\,.
\end{align*}
Under our assumptions, we consider only those $\rho$ such that $f_\rho\in W^{\alpha,\infty}(\Omega)$, and denote the collection of all such $\rho$ by $\pcal$. Let $\fcal=\{f_\rho:\rho\in\pcal\}\subset W^{\alpha,\infty}(\Omega)$. We say that $\{f_1,\cdots,f_T\}\subset\fcal$ is a \emph{$2\delta$-separated set} if for any $i\neq j$,
\begin{align*}
    \|f_i-f_j\|_{L^2(\mathrm{d} \mu)}\ge 2\delta\,.
\end{align*}
Write $f_i=f_{\rho_i}$, where $\rho_i\in\pcal$. We consider the \emph{hypothesis testing problem} generated by the family of distributions $\{\rho_1,\cdots,\rho_T\}$:
\begin{enumerate}
    \item Select an index $J=j$ uniformly at random from the index set $[T]=\{1,\cdots,T\}$.
    \item Given $J=j$, draw a random sample $Z$ from $\rho_j$, that is, $(Z|J=j)\sim\rho_j$.
\end{enumerate}
Let $\bar{\rho}$ denote the joint distribution of $(Z,J)$. Note that the marginal distribution of $Z$ is $\bar{\rho}_Z=\frac{1}{T}\sum_{j=1}^T\rho_j$. Given a sample $Z$, the hypothesis testing problem asks us to determine the random index $J$ from this sample. A \emph{test function} is a mapping $\Psi:\mathcal{Z}\to[T]$, and its error probability is defined by $\bar{\rho}(\Psi(Z)\neq J)$. The hypothesis testing problem is closely related to the minimax lower bound, namely \cite[Proposition 15.1]{Wainwright_2019},
\begin{align*}
    \inf_{\hat{f}}\sup_{f_\rho\in\fcal}\ebb\left[\|\hat{f}-f_\rho\|_{L^2(\mathrm{d} \mu)}^2\right]\ge \delta^2\inf_{\Psi}\bar{\rho}(\Psi(Z)\neq J)\,,
\end{align*}
where the infimum on the right-hand side is taken over all test functions. The difficulty of the hypothesis testing problem depends on the dependence between $Z$ and $J$. If $Z$ and $J$ are independent, the observed sample $Z$ carries no information about the index $J$. 

To measure the dependence between two probability measures, we introduce some information-theoretic concepts. Let $\pbb,\qbb$ be probability measures on a measurable space $\mathcal{X}$ with base measure $\nu$. Since they always admit densities with respect to the base measure $\nu=\frac{1}{2}(\pbb+\qbb)$, we may assume without loss of generality that their density functions are $p,q$, respectively. The commonly used \emph{Kullback--Leibler divergence}\label{sym:first:KL} in information theory is defined by
\begin{align*}
    D(\pbb\|\qbb)=\int_{\mathcal{X}}p(x)\log\frac{p(x)}{q(x)}d\nu(x)\,,
\end{align*}
which quantifies discrepancy between probability measures, although it is not a metric because it is generally not symmetric: $D(\pbb\|\qbb)\neq D(\qbb\|\pbb)$. Two random variables $(Z,J)$ are independent if and only if their joint distribution $\bar{\rho}$ equals the product of their marginals $\bar{\rho}_Z\otimes\bar{\rho}_J$. This motivates measuring the dependence between $Z$ and $J$ through the Kullback--Leibler divergence. Define the \emph{mutual information}\label{sym:first:mutual_info} of $(Z,J)$ by
\begin{align*}
    I(Z;J)=D(\bar{\rho}\|\bar{\rho}_Z\otimes\bar{\rho}_J)\,.
\end{align*}
By the properties of the Kullback--Leibler divergence, we have $I(Z;J)\ge 0$, and $I(Z;J)=0$ if and only if $Z$ and $J$ are independent. 
\begin{lemma}[Fano, {\cite[Proposition 15.12]{Wainwright_2019}}]\label{lem:fano}
    The minimax risk of the function class $\fcal$ satisfies
    \begin{align*}
        \inf_{\hat{f}}\sup_{f_\rho\in\fcal}\ebb\left[\|\hat{f}-f_\rho\|_{L^2(\mathrm{d} \mu)}^2\right]\ge \delta^2\left(1-\frac{I(Z;J)+\log 2}{\log T}\right)\,.
    \end{align*}  
\end{lemma}

By Fano's lemma, it suffices to find an upper bound on $I(Z;J)$ in order to obtain a minimax lower bound. For this purpose, we can use the Yang--Barron inequality.
\begin{lemma}[Yang--Barron, {\cite{YangBarron1999}}]
    Let $\ncal_{KL}(\varepsilon;\pcal)$ be the $\varepsilon$-covering number of $\pcal$ under the square-root KL divergence. Then
    \begin{align*}
        I(Z;J)\le\inf_{\varepsilon>0}\left\{\varepsilon^2+\log \ncal_{KL}(\varepsilon;\pcal)\right\}\,.
    \end{align*}
\end{lemma}

Combining Fano's lemma and the Yang--Barron inequality, we can derive the minimax lower bound in two steps:
\par\noindent{\bf Proof of Theorem~\ref{thm:lowerbound1}.}
    Let $BW^{\alpha,\infty}(\Omega)$ denote the unit ball in $W^{\alpha,\infty}(\Omega)$. When $\mu\sim U(\Omega)$, we have
    \begin{align*}
        \log\ncal(\delta,BW^{\alpha,\infty}(\Omega),\|\cdot\|_{L^2(\mathrm{d} \mu)})\asymp\log\ncal(\delta,BW^{\alpha,\infty}(\Omega),\|\cdot\|_{L^2(\Omega)})\asymp \delta^{-\frac{d}{\alpha}}\,.
    \end{align*}
    Here the second equivalence follows from the Birman--Solomjak theorem \citep{Birman_1967}. Therefore, we may choose a $\delta$-separated set $\{f_1,\cdots,f_T\}$ in $W^{\alpha,\infty}(\Omega)$ such that $\log T\gtrsim \delta^{-\frac{d}{\alpha}}$.

    For $j\in [T]$, let $\rho_j$ denote the distribution of the random vector $y$ conditional on $\{\xb_i\}_{i=1}^n$ when the true function is $f_j$, namely,
    \begin{align*}
        y_i=f_j(\xb_i)+\eta_i,\ \ \eta_i\sim\ncal(0,\sigma^2),\ \ y=(y_1,\cdots,y_n)^\top\,.
    \end{align*}
    Let $\mu^n$ denote the joint distribution of $\{\xb_i\}_{i=1}^n$. Then the joint distribution of $(y,\{\xb_i\}_{i=1}^n)$ is $\rho_j\otimes\mu^n$. By the formula for the Kullback--Leibler divergence between Gaussian measures, for any $j\neq k$,
    \begin{align*}
        D(\rho_j\otimes\mu^n\|\rho_k\otimes\mu^n)=\ebb_x[D(\rho_j\|\rho_k)]=\frac{n}{2\sigma^2}\|f_j-f_k\|_{L^2(\mathrm{d} \mu)}^2\,.
    \end{align*}
    Therefore,
    \begin{align*}
        \log\ncal_{KL}(\varepsilon;\pcal)\le&\log\ncal\left(\varepsilon\sqrt{\frac{2\sigma^2}{n}},BW^{\alpha,\infty}(\Omega),\|\cdot\|_{L^2(\mathrm{d} \mu)}\right)\lesssim \left(\frac{\varepsilon^2}{n}\right)^{-\frac{d}{2\alpha}}\,.
    \end{align*}
    Combining Fano's lemma and the Yang--Barron inequality, we can prove the conclusion by exactly the same argument as in the previous theorem. This completes the proof of Theorem~\ref{thm:lowerbound1}.
\hfill\BlackBox\\[2mm]

The proof of Theorem~\ref{thm:lowerbound2} follows by exactly the same method, and is therefore omitted.

\vskip 0.2in
\bibliography{sample}
\end{document}